\def\year{2019}\relax
\documentclass[letterpaper]{article} 
\usepackage{aaai19}  
\usepackage{times}  
\usepackage{helvet}  
\usepackage{courier}  
\usepackage{url}  
\usepackage{graphicx}  
\usepackage{lipsum}
\usepackage{amssymb,graphicx,mathrsfs,epsfig,epstopdf,amsthm}
\usepackage{algorithm,algorithmic}
\usepackage{amsmath}
\usepackage{amsfonts}
\usepackage{float}
\usepackage{url}

\begin{document}
 
\def \xx {\mathbf{x}}
\def \I {\mathbb{I}}
\def \R {\mathbb{R}}
\def \ww {\mathbf{w}}
\def \E  {\mathbb{E}}
\def \X {\mathcal{X}}
\def \sR {\mathcal{R}}
\def \phii  {\mbox{\boldmath $\phi$}}
\def \alphaa  {\mbox{\boldmath $\alpha$}}
\def \xii  {\mbox{\boldmath $\xi$}}
\def \K {\mathcal{K}}

 \newcommand{\N}{\mathbb{N}}
\newcommand{\Z}{\mathbb{Z}}
\newcommand{\prpo}{\rho+\mu}
\newcommand{\prps}{\rho+{\mu}^2}
\newcommand{\prms}{\rho-{\mu}^2}
\newcommand{\prmo}{\rho-\mu}
\newcommand{\mrpo}{-\rho+\mu}
\newcommand{\mrps}{-\rho+{\mu}^2}
\newcommand{\mrmo}{-\rho-\mu}
\newcommand{\mrms}{-\rho-{\mu}^2}
\newcommand{\first}{\eta<d}
\newcommand{\second}{d\leq {\eta}\leq 1-d}
\newcommand{\third}{\eta>1-d}
\newcommand{\conoo}{\mu < \frac{d(1 - \eta)}{\eta(1-d)}}
\newcommand{\conos}{\mu \geq \frac{d(1 - \eta)}{\eta(1-d)}}

\newtheorem{theorem}{Theorem}
\newtheorem{proposition}[theorem]{Proposition}
\newtheorem{defn}{Definition}
\newtheorem{remark}{Remark}
\newtheorem{lemma}[theorem]{Lemma}
\newtheorem{corr}{Corollary}
\newtheorem{Remark}{Remark}
\newtheorem{assumption}{Assumption}
\newcounter{MYtempeqncnt} 
 
%
\title{Sparse Reject Option Classifier Using Successive Linear Programming}
\author{Kulin Shah, Naresh Manwani\\
kulin.shah@students.iiit.ac.in, naresh.manwani@iiit.ac.in\\
Machine Learning Lab, KCIS\\ IIIT, Hyderabad-500032\\
}
\maketitle
\begin{abstract}
In this paper, we propose an approach for learning sparse reject option classifiers using double ramp loss $L_{dr}$. We use DC programming to find the risk minimizer. The algorithm solves a sequence of linear programs to learn the reject option classifier. We show that the loss $L_{dr}$ is Fisher consistent. We also show that the excess risk of loss $L_d$ is upper bounded by excess risk of $L_{dr}$. We derive the generalization error bounds for the proposed approach. We show the effectiveness of the proposed approach by experimenting it on several real world datasets. The proposed approach not only performs comparable to the state of the art, it also successfully learns sparse classifiers. 
\end{abstract}

\section{Introduction}
\label{Loss}
Standard classification tasks focus on building a classifier which predicts well on future examples. The overall goal is to minimize the number of mis-classifications. However, when the cost of mis-classification is very high, a generic classifier may still suffer from very high risk. In such cases it makes more sense not to classify high risk examples. This choice given to the classifier is called reject option. Hence, the classifiers which can also reject examples are called reject option classifiers. The rejection also has its cost but it is very less compared to the cost of mis-classification. 

For example, making a poor decision based on the diagnostic reports can cost huge amount of money on further treatments or it can be cost of a life \cite{Rocha2011}. If the reports are ambiguous or some rare symptoms are seen which are unexplainable without further investigation, then the physician might choose not to risk misdiagnosing the patient.
In this case, he might instead choose to perform further medical tests, or to refer the case to an appropriate specialist. Reject option classifier may also be found useful in financial services \cite{Rosowsky2013}. Consider a banker looking at a loan application of a customer. He may choose not to decide on the basis of the information available, and ask for a credit bureau score or further recommendations from the stake holders.
Reject option classifiers have been used in wide range of applications from healthcare \cite{btn349,Rocha2011} to text categorization \cite{1234113} to crowd sourcing \cite{Qunwei2017} etc.

Reject option classifier can be viewed as combination of a classifier and a rejection function. The rejection region impacts the proportion of examples that are likely to be rejected, as well as the proportion of predicted examples that are likely to be correctly classified. An optimal reject option classifier is the one which minimizes the rejection rate as well as the mis-classification rate on the predicted examples.

Let ${\cal X} \subseteq \R^d$ be the feature space and ${\cal Y}$ be the label space. For binary classification, we use ${\cal Y} = \{+1,-1\}$. Examples $(\xx,y)$ are generated from unknown joint distribution ${\cal D}$ on the product space ${\cal X} \times {\cal Y}$. A typical {\em reject option classifier} is defined using a decision surface ($f(\xx)=0$) and bandwidth parameter $\rho$ (determines rejection region) as follows:
\begin{equation}
\label{eq:rej-op-classifier}
h_{\rho}(f(\xx)) = 
1.I_{\{f(\xx) > \rho\}}+
0.I_{\{f(\xx)| \leq \rho\}} -1.I_{\{f(\xx)< -\rho\}}
\end{equation}
A reject option classifier can be viewed as two parallel surfaces and the area between them as rejection region.
The goal is to determine both $f$ and $\rho$ simultaneously. The performance of a reject option classifier is measured using $L_{d}$ loss function defined as:
\begin{equation}
L_{d}(yf(\xx), \rho) = 1.I_{\{yf(\xx) < -\rho\}} + d.I_{\{|f(\xx)| \leq \rho\}}
\end{equation}
where $d$ is the cost of rejection. If $d=0$, then $f(.)$ will always reject. If $d \geq 0.5$, then $f(\xx)$ will never reject, since the cost of random labeling is $0.5$. Thus,
$d$ is chosen in the range $(0,0.5)$.
$h_\rho(f(\xx))$ (described in equation.~\ref{eq:rej-op-classifier}) has been shown to be infinite sample consistent with respect to the generalized Bayes classifier \cite{Yuan:2010}. A reject option classifier is learnt by minimizing the risk which is the expectation of $L_{d}$ with respect to the joint distribution ${\cal D}$. The risk under $L_{d}$ is minimized by {\em generalized Bayes discriminant} $f_d^*(\xx)$ \cite{Chow:2006}, which is
\begin{equation}\label{gbd}
f_d^*(\xx) = 1.\I_{\{\eta(\xx) > 1-d\}}
+0.\I_{\{d\leq \eta(\xx) \leq 1-d\}}
-1.I_{\{\eta(\xx) < d\}}
\end{equation}
where $\eta(\xx)=P(y=1|\xx)$. However, in general we do not know ${\cal D}$. But, we have the access to a finite set of examples drawn from ${\cal D}$ called training set. We find the reject option classifier by minimizing the empirical risk. Minimizing the empirical risk under $L_{d}$ is computationally hard. To overcome this problem, convex surrogates of $L_{d}$ have been proposed. Generalized hinge based convex loss has been proposed for reject option classifier \cite{Bartlett:2008}. The paper describes an algorithm for minimizing $l_2$ regularized risk under generalized hinge loss. Wegkamp et.al 2011 \cite{wegkamp2011} propose sparse reject option approach by minimizing $l_1$ regularized risk under generalized hinge loss. In both these approaches \cite{Bartlett:2008,wegkamp2011}, first a classifier is learnt based on risk minimization under generalized hinge loss and then a rejection threshold is learnt. Ideally, the classifier and the rejection threshold should be found simultaneously. This approach might not give the optimal parameters. Also, a very limited experimental results are provided to show the effectiveness of the proposed approaches \shortcite{wegkamp2011}. A cost sensitive convex surrogate for $L_{d}$ called double hinge loss has been proposed in \cite{Grandvalet2008}. The double hinge loss remains an upper bound to $L_{d}$ provided $\rho \in \bigg( \frac{1-H(d)}{1-d},\frac{H(d)-d}{d}\bigg)$, which is very strict condition. So far, the approaches proposed learn a threshold for rejection along with the classifier. However, in general, the rejection region may not be symmetrically located near the classification boundary. A generic convex approach has been proposed which simultaneously learns the classifier as well as the rejection function \cite{Cortes2016}. The main challenge with the convex surrogates is that they are not constant even in the reject region in contrast to $L_{d}$ loss. 
Sousa and Cardoso \cite{Sousa:2013} model reject option classification as ordinal regression problem. It is not clear whether treating rejection as a separate class leads to a good approximation simply because training data does not contain rejection as a class label. Moreover, classification consistency of this approach is not known in the reject option context. 
A non-convex formulation for learning reject option classifier using logistic function is proposed in Fumera and Roli \shortcite{Fumera2002}. However, theoretical guarantees for the approach are not known. Also, a very limited set of experiments are provided in support of the approach. A bounded non-convex surrogate called {\em double ramp loss} $L_{dr}$ is proposed in Manwani \textit{et al.} \shortcite{Manwani15}. A regularized risk minimization algorithm was proposed with $l_2$ regularization \cite{Manwani15}. The approach proposed shown to have interesting geometric properties and robustness to the label noise. However, statistical properties of $L_{dr}$ (Fisher consistency, generalization error etc.) are not studied so far. 
Also, $l_2$ regularization based approach does not learn sparse classifiers. 
\subsection{Our Contributions}
In this paper, we propose a sparse reject option classifier learning algorithm using double ramp loss. By sparseness, we mean that the number of support vectors needed to express the classifier are small. Our contributions in this work are as follows.
\begin{itemize}
\item We propose a difference of convex (DC) programming \cite{ThiHoaiAn:1997} based algorithm to learn sparse reject option classifier. The final algorithm turns out to be solving successive linear programs.
\item We also establish statistical properties for double ramp loss function. We show that the double ramp loss function is Fisher consistent. Which means that generalized Bayes classifier minimizes the population risk under $L_{dr}$. We also show that the excess risk under loss $L_{dr}$ upper bounds the excess risk under loss $L_d$.
\item We derive the generalization error bounds for the proposed approach.
\item We also show experimentally that the proposed approach performs comparable to the other state of the art approaches for reject option classifier. Our approach learns sparser classifiers compared to all the other approaches. We also show experimentally that the proposed approach is robust against label noise.
\end{itemize}

The rest of the paper is organized as follows. We discuss the proposed method and algorithm in section~\ref{sec:approach}. In section~\ref{sec:analysis}, we provide the theoretical results for $L_{dr}$. The experiments are given in section~\ref{sec:exp}. We conclude the paper with some remarks in section~\ref{sec:conclusions}.

\section{Proposed Approach}\label{sec:approach}
We propose a new algorithm for learning reject option classifier which minimizes the $l_1$-regularized risk under double ramp loss function $L_{dr}$ \cite{Manwani15}. $L_{dr}$ is a non-convex surrogate of $L_{d}$ as follows. 
\begin{equation}
\begin{aligned}
 L_{dr}&(t,\rho) = \frac{d}{\mu}\Big{[}\big{[}\mu-t
+\rho\big{]}_+ - \big{[}-\mu^2-t+\rho\big{]}_+\Big{]} \\
&+\frac{(1-d)}{\mu}\Big{[}\big{[}\mu
-t-\rho\big{]}_+ - \big{[}-\mu^2-t-\rho\big{]}_+\Big{]}
\end{aligned}
\end{equation}
where $\mu$ is the slope of the loss in linear region, $[a]_+=\max(0,a)$ and $t=yf(\xx)$. Note that $L_{dr}$ depends on specific choice of $\mu$. Also, for a valid reject region, we want $\rho \geq \frac{1}{2}\mu(1+\mu)$. Figure~\ref{DRLoss} shows the plot of $L_{dr}$ for different values of $\mu$.
\begin{figure}
\begin{center}
\includegraphics[scale=0.45]{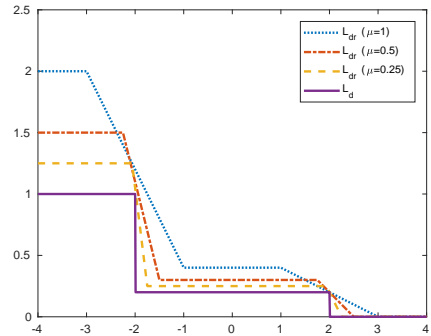}
\caption{$L_{d}$ vs. Double ramp loss $L_{dr}$ ($d$=0.2, $\rho = 2$).}
\label{DRLoss}
\end{center}
\end{figure}

\subsection{Sparse Double Ramp SVM (SDR-SVM)}
Let $S=\{(\xx_1,y_1),\ldots,(\xx_N,y_N)\}$ be the training set where $(\xx_i,y_i) \in \X\times \{+1,-1\},\;i=1\ldots N$. Let the reject option classifier be of the form $f(\xx)=h(\xx) + b$.
Let $\K:\X \times \X \rightarrow \R_+$ be a Mercer kernel (continuous, symmetric and positive semi-definite) to produce nonlinear classifiers. Let ${\cal H}_{\K}$ be the reproducing kernel Hilbert space (RKHS) induced by the Mercer kernel $\K$ with the norm $\|.\|_{\K}$ \cite{aronszajn1950theory}. 
To learn sparse reject option classifier, we use $l_1$ regularization term. Thus, we find the classifier as solving following optimization problem.
$$\min_{h\in {\cal H}_{\K}^+,b,\rho}\;\;\; \lambda \|h\|_1 + \sum_{i=1}^N L_{dr}(y_if(\xx_i),\rho)$$
However, the optimal $h$ lies in a finite dimensional subspace ${\cal H}_{\K,S}^+$ of ${\cal H}_{\K}$ \cite{Scholkopf:2001}.
${\cal H}_{\K,S}^+ = \left\{\sum_{i=1}^Ny_i\alpha_i\K(\xx_i,.)\;|\;[\alpha_1,\ldots,\alpha_N] \in \R^N_+\right\}$. Given $h \in {\cal H}_{\K,S}^+$, the $l_1$ regularization is defined as $\Omega(h) = \sum_{i=1}^N \alpha_i$ for $h(\xx) = \sum_{i=1}^N y_i\alpha_i\K(\xx_i,\xx)$ \cite{817991,bradley2000massive,Wu:2005}. Thus, the sparse double ramp SVM can be learnt by minimizing following $l_1$ regularized risk.
\begin{align}
\label{eq-sdr-svm}
J(\Theta)=\lambda \sum_{i=1}^N \alpha_i +  \frac{1}{N}\sum_{i=1}^N L_{dr}(y_i f(\xx_i),\rho)
\end{align}
where $f(\xx_i) = \sum_{j=1}^N y_j\alpha_j\K(\xx_i,\xx_j) + b$. $\Theta=(\alphaa,b,\rho)$. We see that $J$ is a non-convex function. However, $J$ can decomposed as a difference of two convex functions $Q_1$ and $Q_2$ as $J(\Theta)= Q_1(\Theta)-Q_2(\Theta)$,
where
\begin{align*}
Q_1(\Theta) & = \lambda\sum_{i=1}^N \alpha_i + \frac{1}{N\mu}\sum_{i=1}^N \bigg[ d\big[\mu-y_if(\xx_i)+\rho\big]_+ \\ 
& + (1-d)\big[\mu-y_if(\xx_i)-\rho\big]_+ \bigg]\\
Q_2(\Theta) & =  \frac{1}{N\mu}\sum_{i=1}^N \bigg[ d\big[-\mu^2-y_if(\xx_i)+\rho\big]_+ \\ 
& +(1-d)\big[-\mu^2-y_if(\xx_i)-\rho\big]_+\bigg]
\end{align*}
To minimize such a function which can be expressed as difference of two convex functions, we can use difference of convex (DC) programming. In this case, DC programming guarantees to find a local optima of the objective function \cite{ThiHoaiAn:1997}. The simplified DC algorithm uses the convexity property of $Q_2(\Theta)$ and finds an upper bound on $J(\Theta)$ as $J(\Theta) \leq B(\Theta, \Theta^{(l)})$,
where 
$$B(\Theta, \Theta^{(l)}):= Q_1(\Theta) - Q_2(\Theta^{(l)}) - (\Theta - \Theta^{(l)})^T \nabla Q_2(\Theta^{(l)})$$
$\Theta^{(l)}$ is the parameter vector after $(l)^{th}$ iteration, $\nabla Q_2(\Theta^{(l)})$ is a sub-gradient of $Q_2$ at $\Theta^{(l)}$. $\Theta^{(l+1)}$ is found by minimizing $B(\Theta,\Theta^{(l)})$. 
Thus,
$$J(\Theta^{(l+1)})  \leq  B(\Theta^{(l+1)}, \Theta^{(l)})  \leq B(\Theta^{(l)}, \Theta^{(l)})= J(\Theta^{(l)})$$
Thus, the DC program reduces the value of $J(\Theta)$ in every iteration. Now, we will derive a DC algorithm for minimizing $ J(\Theta)$. Given $\Theta^{(l)}$, we find $\Theta^{(l+1)} \in  {\arg\min}_{\Theta} \;B(\Theta, \Theta^{(l)})
= {\arg\min}_{\Theta} \;Q_1(\Theta) - \Theta^T \nabla Q_2(\Theta^{(l)})
$. We use $\nabla Q_2(\Theta^{(l)})$ as:
\begin{align*}
\nabla Q_2(\Theta^{(l)})&=
-\sum_{i=1}^N \begin{pmatrix}
\frac{d\beta_i'^{(l)}+(1-d)\beta_i''^{(l)}}{\mu N}y_1y_i\K(\xx_1,\xx_i)\\
\vdots\\
\frac{d\beta_i'^{(l)}+(1-d)\beta_i''^{(l)}}{\mu N}y_Ny_i\K(\xx_N,\xx_i)\\
\frac{d\beta_i'^{(l)}+(1-d)\beta_i''^{(l)}}{\mu N}y_i\\
-\frac{d\beta_i'^{(l)}-(1-d)\beta_i''^{(l)}}{\mu N}
\end{pmatrix} 
\end{align*}
where 
\begin{align*}
\beta_i'^{(l)}&=\I_{\{y_if^{(l)}(\xx_i)\leq \rho^{(l)} - \mu^2\}};\;i=1\ldots N\\
\beta_i''^{(l)}&=\I_{\{y_if^{(l)}(\xx_i)\leq -\rho^{(l)} - \mu^2\}};\;i=1\ldots N
\end{align*}
Note that $f^{(l)}(\xx) = \sum_{i=1}^N\alpha_i^{(l)}y_i\K(\xx_i,\xx) + b^{(l)}$. The new parameters $\Theta^{(l+1)}$ are found by minimizing $B(\Theta,\Theta^{(l)})$ subject to $\rho \geq \frac{1}{2}\mu(1+\mu)$. Which becomes
\begin{align*}
&\min_{\alphaa,b,\rho,\xii',\xii''}\lambda \sum_{i=1}^N\alpha_i +\frac{1}{N\mu}\sum_{i=1}^N \big( d\xi_i' + (1-d) \xi_i''\big) \\ &\;\;\;\;\;\;\;+\frac{d}{N\mu}\sum_{i=1}^N\beta_i'^{(l)}\big[ y_i\big(\sum_{j=1}^N\alpha_j y_j\K(\xx_j,\xx_i) + b\big)-\rho\big]\\
&\;\;\;\;\;\;\;+\frac{1-d}{N\mu}\sum_{i=1}^N\beta_i''^{(l)}\big[ y_i\big(\sum_{j=1}^N\alpha_j y_j\K(\xx_j,\xx_i) + b\big)+\rho\big]\\
&s.t.\begin{cases}
y_i\big(\sum_{j=1}^N\alpha_j y_j\K(\xx_j,\xx_i) + b\big)\geq \rho+\mu-\xi_i'\; \forall i\\
y_i\big(\sum_{j=1}^N\alpha_j y_j\K(\xx_j,\xx_i) + b\big)\geq -\rho+\mu-\xi_i'' \; \forall i\\
\alpha_i,\xi_i',\xi_i'' \geq 0\;\; \forall  i\;\;\;\;\rho \geq \frac{1}{2}\mu(1+\mu)
\end{cases}
\end{align*}
Thus, $B(\Theta,\Theta^{(l)})$ can be minimized by solving a linear program. Thus, the algorithm solves a sequence of linear programs to learn a sparse reject option classifier. The complete approach is described in Algorithm~\ref{algo1}. Convergence guarantee of this algorithm follows from the convergence of DC algorithm given in \cite{ThiHoaiAn:1997}.
The final learnt classifier is represented as $f(\xx)=h(\xx)+b$ and $\rho$.
\begin{algorithm}[t]
\caption{Sparse Double Ramp SVM (SDR-SVM)}
\label{algo1}
\begin{algorithmic}
\STATE {\bf Input: }$S=\{(\xx_1,y_1),\ldots,(\xx_N,y_N)\},\;\epsilon>0,\;d\in (0,0.5),\;\mu\in (0,1],\;\lambda>0$\;
\STATE {\bf Output: }$\alphaa^*,b^*,\rho^*$\;
\STATE {\bf Initialize: }$l=0$, $\alphaa^{(0)},b^{(0)},\rho^{(0)}$\;
\WHILE{($J(\Theta^{(l)})-J(\Theta^{(l+1)})>\epsilon$)}
\FOR{$i = 1$ to $N$}
\STATE $\beta_i'^{(l)}=\I_{\{y_if^{(l)}(\xx_i)\leq \rho^{(l)} - \mu^2\}}$\;
\STATE $\beta_i''^{(l)}=\I_{\{y_if^{(l)}(\xx_i)\leq -\rho^{(l)} - \mu^2\}}$
\ENDFOR
\STATE $\alphaa^{(l+1)},b^{(l+1)},\rho^{(l+1)} = {\arg\min}_{\Theta}\;B(\Theta,\Theta^{(l)})$
\ENDWHILE
\end{algorithmic}
\end{algorithm}

\section{Analysis}
\label{sec:analysis}
In this paper, we are proposing an algorithm based on $L_{dr}$. We first need to ensure that minimizer of the population risk under $L_{dr}$ is minimized by the generalized Bayes classifier $f_d^*$ (defined in eq.(\ref{gbd})). This property is called Fisher consistency or classification calibrated-ness.

\begin{theorem}{\bf Fisher Consistency of $L_{dr}$}
The generalized Bayes discriminant function $f_{d}^{*}(\xx)$ (described in eq.~(\ref{gbd})) minimizes the risk
  \begin{equation*}
	\sR_{dr}(f,\rho)=\E\big[ L_{dr}(yf(\xx),\rho)\big]
  \end{equation*}
 over all measurable functions $ f $.
\end{theorem}
The proof of this theorem is provided in Appendix~A. To approximate the optimal classifier, Fisher consistency is the minimal requirement for the loss function. 

\subsection{Excess Risk Bound}
We will now derive the bound on the excess risk $(\sR_d(f,\rho) - \sR_d(f_d^*,\rho_d^*))$ in terms of the excess risk under $L_{dr}$ where $\sR_{d}(f, \rho) = \mathbb{E}[L_{d}(yf(\xx), \rho)]$. We know that $L_{d}(f(\xx), \rho) \leq L_{dr}(f(\xx), \rho),\;\forall \xx \in \X,\;\forall f$. This relation remains preserved when we take expectations both side, means $\sR_{d}(f, \rho) \leq \sR_{dr}(f,\rho)$. This relation is also true for excess risk. To show that, We first define the following terms. Let $\eta(\xx) = P(y = 1 | \xx)$ and $z = f(\xx)$. We define following terms.
\begin{equation*}
\begin{aligned}
\xi(\eta) &:=  \eta \I_{\{\first\}}+d\I_{\{\second\}}+(1-\eta)\I_{\{\third\}} \\
H(\eta) &:= \inf_{z, \rho}\;\;{\eta}L_{dr}(z,\rho) + (1-\eta)L_{dr}(-z,\rho)\\
&= \eta( 1+\mu )\I_{\{\first\}} + d( 1+\mu )\I_{\{\second\}}\\ &+ (1-\eta)(1+\mu)\I_{\{\third\}}
\end{aligned}
\end{equation*}
We know that $\sR_{d}^{*}=\E [\xi(\eta)]$ and $\sR_{dr}^{*}=\E[H(\eta)]$. Furthermore, we define
\begin{equation*}
\begin{aligned}
{\xi}_{-1}(\eta) &:= \eta - {\xi}(\eta) \\
{\xi}_{r}(\eta) &:= d - {\xi}(\eta) \\
{\xi}_{1}(\eta) &:= (1-\eta) - {\xi}(\eta) \\
H_{-1}(\eta) &:= \inf_{z<-\rho}\;\;{\eta}L_{dr}(z,\rho) + (1-\eta)L_{dr}(-z,\rho)  \\
H_{r}(\eta) &:= \inf_{|z| \leq \rho}\;\;{\eta}L_{dr}(z,\rho) + (1-\eta)L_{dr}(-z,\rho)  \\
H_{1}(\eta) &:= \inf_{z>\rho}\;\;{\eta}L_{dr}(z,\rho) + (1-\eta)L_{dr}(-z,\rho) 
\end{aligned}
\end{equation*}
We observe the following relationship.
\begin{proposition}
\begin{equation*}
\begin{aligned}
{\xi}_{-1}(\eta) &\leq {H_{-1}}(\eta) - H(\eta)\\
{\xi}_{r}(\eta) &\leq {H_{r}}(\eta) - H(\eta)  \\ 
{\xi}_{1}(\eta) &\leq {H_{1}}(\eta) - H(\eta)  \\
\end{aligned}
\end{equation*}
\end{proposition}
The proof of the Proposition~2 is given in Appendix~B. Now we prove that the excess risk of $L_{d}$ loss is bounded by excess risk of $L_{dr}$ using above proposition.

\begin{theorem}
\label{thm-excess-risk}
For any measurable function $f:\X\rightarrow\mathbb{R}$,
\begin{equation*}
\sR_{d}(f, \rho)-\sR_{d}(f_d^*, \rho_d^*) \leq \sR_{dr}(f,\rho)-\sR_{dr}(f_d^*,\rho_d^*)
\end{equation*}
\end{theorem}
\begin{proof}
We know that 
$$\sR_{d}(f, \rho)=\E[{\eta}{\I_{\{f<-\rho\}}} + {d}{\I_{\{-\rho\leq f\leq \rho\}}} + (1-\eta){\I_{\{f>\rho\}}}]$$
and $\sR_{dr}(f,\rho)=\E[r_{\eta}(f)]$ where $r_{\eta}(f(\xx))= \E_{y|\xx}[L_{dr}(yf(\xx),\rho)] = \eta L_{dr}(f(\xx), \rho) + (1 - \eta)L_{dr}(-f(\xx), \rho)$ . Therefore, 
\begin{equation*}
\begin{aligned}
&\sR_{d}(f, \rho)-\sR_{d}(f_d^*, \rho_d^*)\\
&= \E\big[{\eta}{\I_{\{f<-\rho\}}} + {d}{\I_{\{|f|\leq \rho\}}} + (1-\eta){\I_{\{f>\rho\}}}\big] -\E\big[\xi(\eta)\big] \\
&= \E\big[{\xi_{-1}(\eta)}{\I_{\{f<-\rho\}}} + {\xi_{r}(\eta)}{\I_{\{-\rho\leq f\leq \rho\}}} + {\xi_{1}(\eta)}{\I_{\{f>\rho\}}}\big]
\end{aligned}
\end{equation*}
Using Proposition~1, we will get
\begin{equation*}
\begin{aligned}
&\sR_{d}(f, \rho)-\sR_{d}(f_d^*, \rho_d^*) \;\leq \; \E\big[(H_{-1}(\eta) - H(\eta)){\I_{\{f<-\rho\}}} \\
& + (H_{r}(\eta) - H(\eta)){\I_{\{-\rho\leq f\leq \rho\}}} + 	(H_{1}(\eta) - H(\eta)){\I_{\{f>\rho\}}}\big] \\
& \leq \;\E\big[H_{-1}(\eta){\I_{\{f<-\rho\}}} + H_{r}(\eta){\I_{\{-\rho\leq f\leq \rho\}}}\\
&\;\;\;\;+ H_{1}(\eta){\I_{\{f>\rho\}}} - H(\eta)\big]  \\
& \leq \;\E[r_{\eta}(f) - H(\eta)] \leq \; \sR_{dr}(f,\rho) - \sR_{dr}(f_d^*,\rho_d^*)
\end{aligned}
\end{equation*}
\end{proof}
Hence, excess risk under $L_d$ is upper bounded by excess risk under $L_{dr}$. From Theorem~\ref{thm-excess-risk}, we need to bound $\sR_{dr}(f, \rho) - \sR_{dr}(f_d^*, \rho_d^*)$ in order to bound $\sR_{d}(f, \rho) - \sR_{d}(f_d^*, \rho_d^*)$. We thus need an error decomposition for $\sR_{dr}(f, \rho) - \sR_{dr}(f_d^*, \rho_d^*)$.

\subsection{Error Decomposition of $\sR_{dr}(f, \rho) - \sR_{dr}(f_d^*, \rho_d^*)$}
The decomposition for RKHS based regularization schemes is well established \cite{Cucker:2007:LTA:1214096}. To understand the details, consider the $l_2$ regularized empirical risk minimization with $L_{dr}$. For $S=\{(\xx_1,y_1),\ldots,(\xx_N,y_N)\}$ and $\lambda_2 >0$, let $f_{\lambda_{2}, S}^*=h_{\lambda_{2}, S}^*+ b_{\lambda_{2}, S}^*$ where
\begin{align}
\label{l2-emp-risk-min}
(h_{\lambda_{2}, S}^*, b_{\lambda_{2}, S}^*,\rho_{\lambda_{2}, S}^*) &= \arg\min_{h\in {\cal H}_{\K}, b, \rho} \; \frac{\lambda_{2}}{2} \| h \|^2_{\K} + \hat{\cal R}_{dr}(f,\rho)
\end{align}
Note that $\hat{\sR}_{dr}$ denotes the empirical risk under double ramp loss. In this case, we observe the following decomposition.
\begin{align}
\label{error-decomp-l2}
\nonumber &\sR_{dr}(f_{\lambda_{2}, S}^*,\rho_{\lambda_{2}, S}^*)-\sR_{dr}(f_{d}^*, \rho_d^*)\leq \mathcal{A}(\lambda_2) + \sR_{dr}(f_{\lambda_{2}, S}^*,\rho_{\lambda_{2}, S}^*) \\
&-\hat{\sR}_{dr}(f_{\lambda_{2}, S}^*,\rho_{\lambda_{2}, S}^*)+\hat{\sR}_{dr}(f_{\lambda_{2}}^*,\rho_{\lambda_{2}}^*) -\sR_{dr}(f_{\lambda_{2}}^*,\rho_{\lambda_{2}}^*)
\end{align}
where $\hat{\sR}_{dr}(f,\rho)$ is the empirical risk of $(f,\rho)$ under double ramp loss. $f_{\lambda_2}^*=h_{\lambda_2}^*+b_{\lambda_2}^*$ and $\rho^*_{\lambda_2}$ are defined as follows.
\begin{align}
\label{eq-l2-risk-minimize}
(h_{\lambda_{2}}^*, b_{\lambda_{2}}^*,\rho_{\lambda_{2}}^*)& = \arg\min_{h\in {\cal H}_{\K}, b, \rho} \; \frac{\lambda_{2}}{2} \| h \|^2_{\K} + {\cal R}_{dr}(f,\rho)
\end{align}
$\mathcal{A}(\lambda_2)$ measures the approximation power in RKHS $\K$ and is defined as follow.
\begin{equation}
\label{eq-approx-powr}
\mathcal{A} (\lambda_2) = \inf_{h\in {\cal H}_{\K}, b, \rho} \; \frac{\lambda_2}{2} \| h \|^2_{\K} + \sR_{dr} (h+b, \rho) - \sR_{dr} (f_{d}^*, \rho_d^*)  \;\;\; \forall \lambda_2 > 0
\end{equation}
The error decomposition in eq.(\ref{error-decomp-l2}) is easy to derive once we know that both $h^*_{\lambda_2}$ and $h^*_{\lambda_2,S}$ lie in the same function space. However, this doe not hold true in case of SDR-SVM proposed in this paper. It happens because the error analysis becomes difficult due to the data dependent nature of ${\cal H}_{\K}^+$. We use the techniques discussed in \cite{Wu:2005,JMLR:v15:huang14a}. We establish the error decomposition of SDR-SVM using the error decomposition (\ref{error-decomp-l2}) with the help of $f^*_{\lambda_2,S}$. We first characterize some properties of $f^*_{\lambda_2,S}, \rho^*_{\lambda_2,S}$. 
Note that from here onwards, we assume $\mu=1$
(slope parameter in the loss function $L_{dr}$).

\begin{proposition}
\label{prop-reg-bound}
For any $\lambda_{2} > 0$, ${f}_{ \lambda_{2}, S}^* =  (h_{\lambda_{2}, S}^*, b_{\lambda_{2}, S}^*, \rho^*_{\lambda_2,S})$ is given by eq.(\ref{l2-emp-risk-min}). Then,
\begin{equation*}
\begin{aligned}
\Omega({h}_{\lambda_{2}, S}^*) \leq \frac{1}{\lambda_{2}} \hat{\sR}_{dr}({f}_{ \lambda_{2}, S}^*, \rho^*_{\lambda_2,S}) + \| h_{\lambda_{2}, S}^* \|^2_{\K} 
\end{aligned}
\end{equation*}
\end{proposition}
The proof of this proposition is skipped here and is provided in Appendix C.

\subsection{Error Decomposition for SDR-SVM}
We now find the error decomposition for SDR-SVM. We define the sample error as below, 
\begin{align*}
\label{eq-sample-error}
\mathcal{S}(N, \lambda_1, \lambda_2)& = \big( \sR_{dr} (f_{ \lambda_1, S }^*, \rho_{\lambda_1, S}^*) - \hat{\sR}_{dr}(f_{\lambda_1, S}^*, \rho_{\lambda_1, S}^*) \big)\\
&+ (1 + \psi)  \big( \hat{\sR}_{dr}(f_{\lambda_2}^*, \rho_{\lambda_2}^*) - \sR_{dr}( f_{\lambda_2}^*, \rho_{\lambda_2}^* ) \big)
\end{align*}
where $(f_{ \lambda_1, S }^*,\rho_{ \lambda_1, S }^*)$ is a global minimizer of optimization problem in eq.(\ref{eq-sdr-svm}) and $(f^*_{\lambda_2}, \rho^*_{\lambda_2})$ is a global minimizer of problem (\ref{eq-l2-risk-minimize}). Also, $\psi=\frac{\lambda_1}{\lambda_2}$. Following theorem gives the error decomposition for SDR-SVM.
\begin{theorem}
\label{thm:err-decomp-l1}
For $0 < \lambda_1 \leq \lambda_2 \leq 1$, let $\psi = \frac{\lambda_1}{\lambda_2}$. Then,
\begin{align*}
\sR_{dr} (f_{ \lambda_1, S }^*, \rho_{\lambda_1, S}^*) - \sR_{dr} (f_{d}^*, \rho_{d}^*) + \lambda_1 \Omega (h_{\lambda_1, S}^*)\\
\leq \psi \sR_{dr} (f_{d}^*, \rho_d^*)+ \mathcal{S}(N, \lambda_1, \lambda_2) + 2 \mathcal{A}(\lambda_2)
\end{align*}
where $\mathcal{A}(\lambda_2)$ is the approximation error defined by eq.(\ref{eq-approx-powr}). 
\end{theorem}
Proof of above theorem is provided in the Appendix~D. Using Theorem~\ref{thm:err-decomp-l1}, the generalization error of SDR-SVM is estimated by bounding ${\cal S}(N,\lambda_1,\lambda_2)$ and ${\cal A}(\lambda_2)$. 

\subsection{Generalization Error of SDR-SVM}
We expect that the sample error ${\cal S}(N,\lambda_1,\lambda_2)$ tends to zero with certain rate as $N$ tends to infinity. This can be understood by the convergence of the sample mean to its expected value. Also, we will have following assumption on ${\cal A}(\lambda_2)$.  
\begin{assumption}
\label{assump1}
For any $0 < \beta \leq 1$ and $c_{\beta} > 0$, the approximation error satisfies 
\begin{equation}
\label{eq-assumption}
\mathcal{A} (\lambda_2) \leq c_{\beta} \lambda^{\beta}  \;\;\;\forall \lambda_2 > 0
\end{equation}
\end{assumption}
This is a standard assumption in the literature of learning theory \cite{Cucker:2007:LTA:1214096}.

\begin{theorem}
Suppose that Assumption \ref{assump1} holds
for any $0 < \beta \leq 1$. Take $\lambda_1 = N^{-\frac{\beta}{4 \beta + 2}}$ and $(f_{\lambda_1, S}^*, \rho_{\lambda_1, S}^*)$ is the optimal solution of SDR-SVM. Then for any $0 \leq \delta \leq 1$, there holds
\begin{equation}
\sR_{d}(f_{\lambda_1, S}^*, \rho_{\lambda_1, S}^*) - \sR_{d}(f_{d}^*, \rho_d^*) \leq \tilde{c} \left(  \text{log} \frac{4}{\delta} \right)^{1/2} N^{-\frac{\beta}{4\beta + 2}}
\end{equation}
with probability at least $1 - \delta$. Here $\tilde{c}$ is a constant independent of $\delta$ or $N$.
\end{theorem}
Proof of this theorem is provided in Appendix~E. It uses the concentration bounds results discussed in \cite{Bartlett:2003:RGC:944919.944944}.

\section{Experiments}
\label{sec:exp}
In this section, we show the effectiveness of approach on several datasets. We report experimental results on five datasets (``Ionosphere", ``Parkinsons", ``Heart", ``ILPD" and ``Pima Indian Diabetes") available on UCI machine learning repository \cite{Lichman:2013}.

\begin{figure*}[t!]
\begin{center}
\begin{tabular}{ccc}
\includegraphics[scale=0.3]{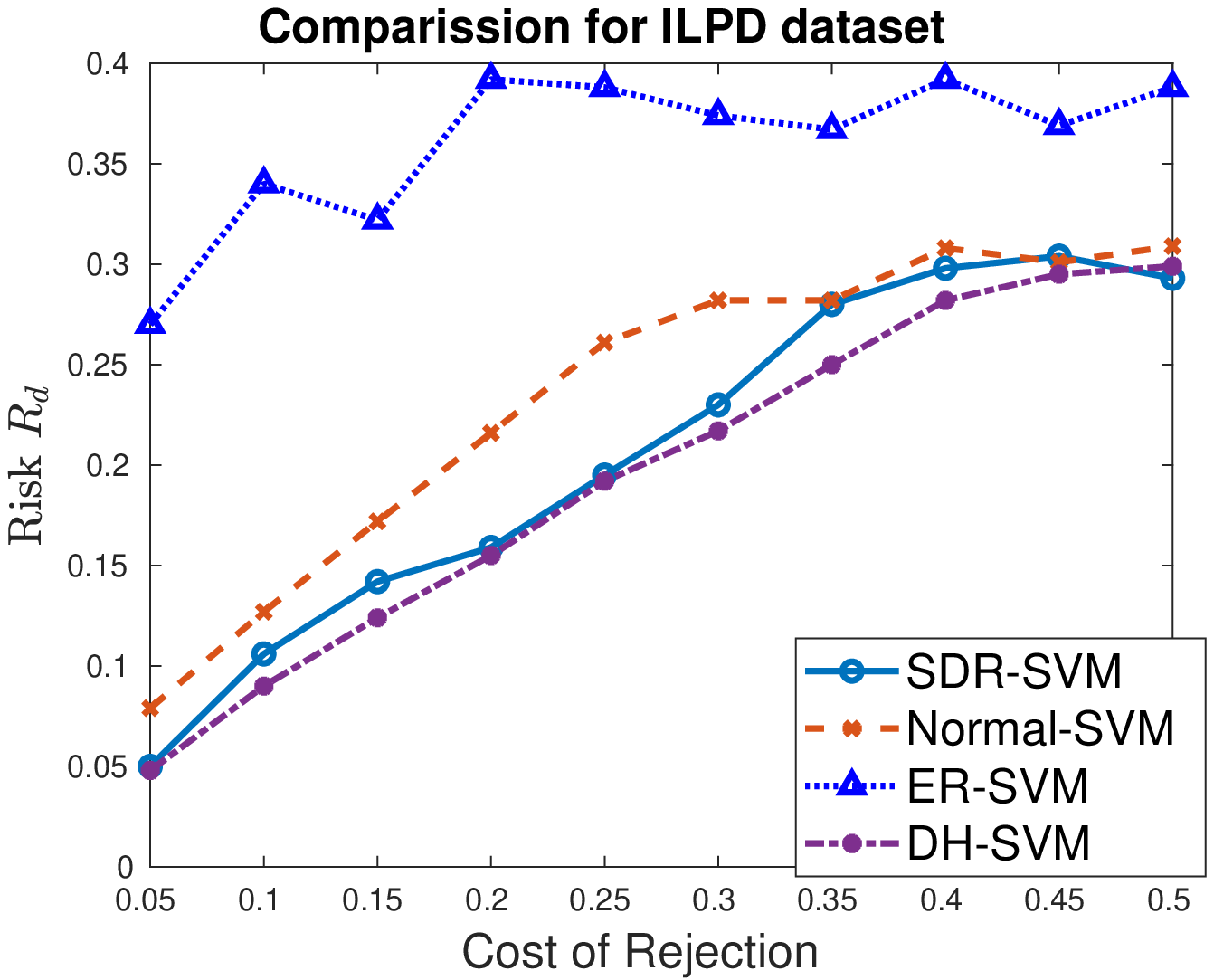}&
\includegraphics[scale=0.3]{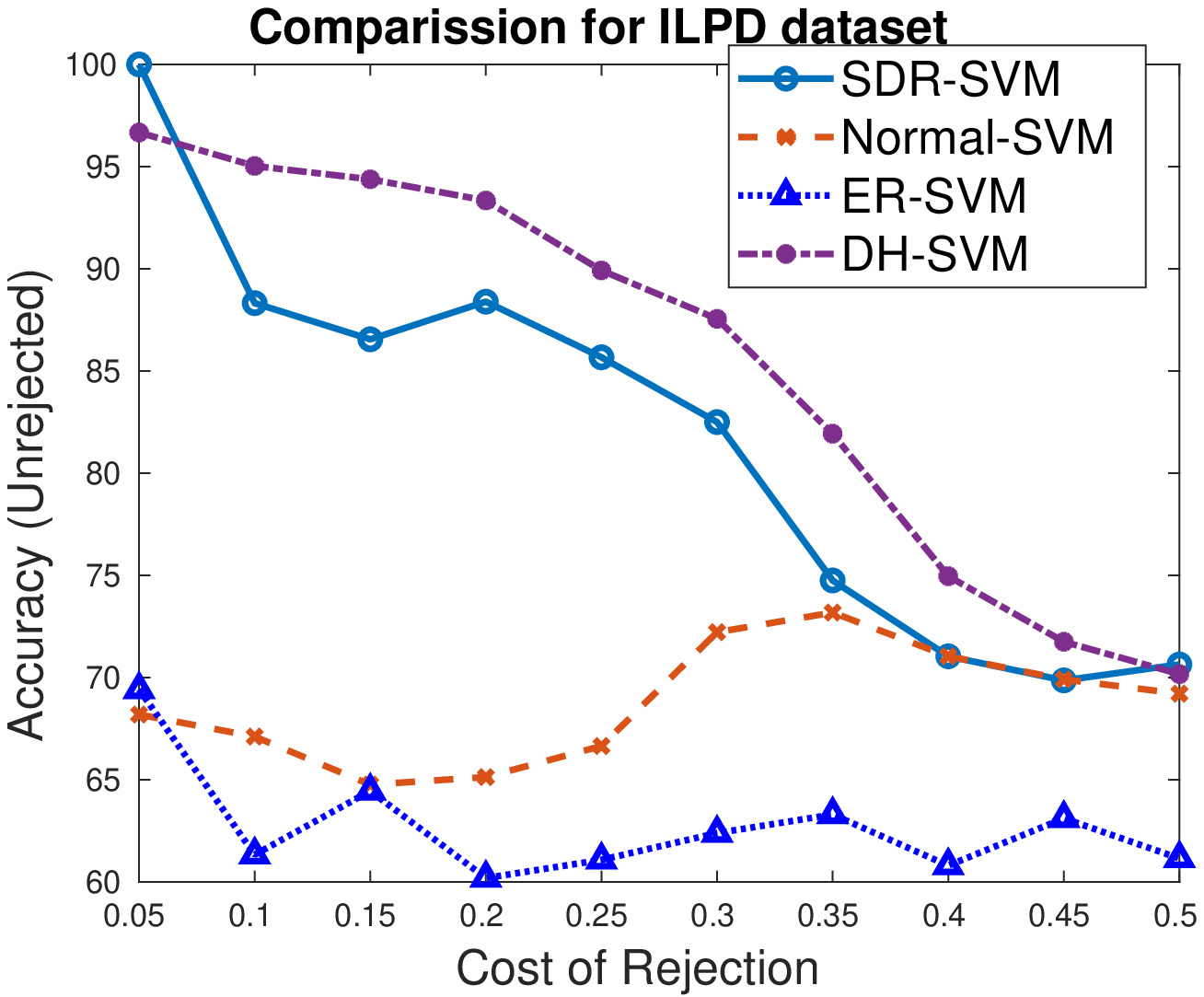} & \includegraphics[scale=0.3]{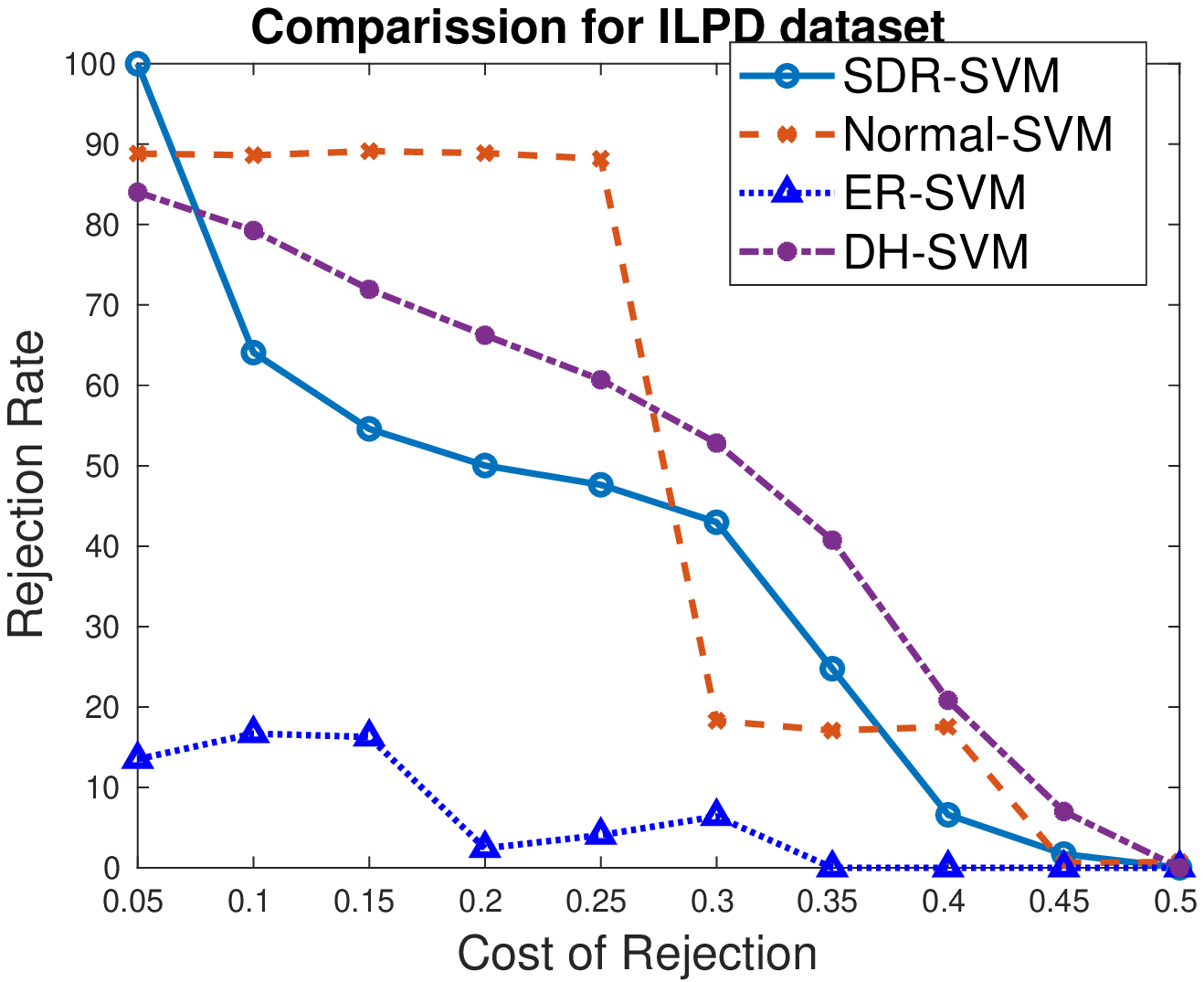} \\
\includegraphics[scale=0.3]{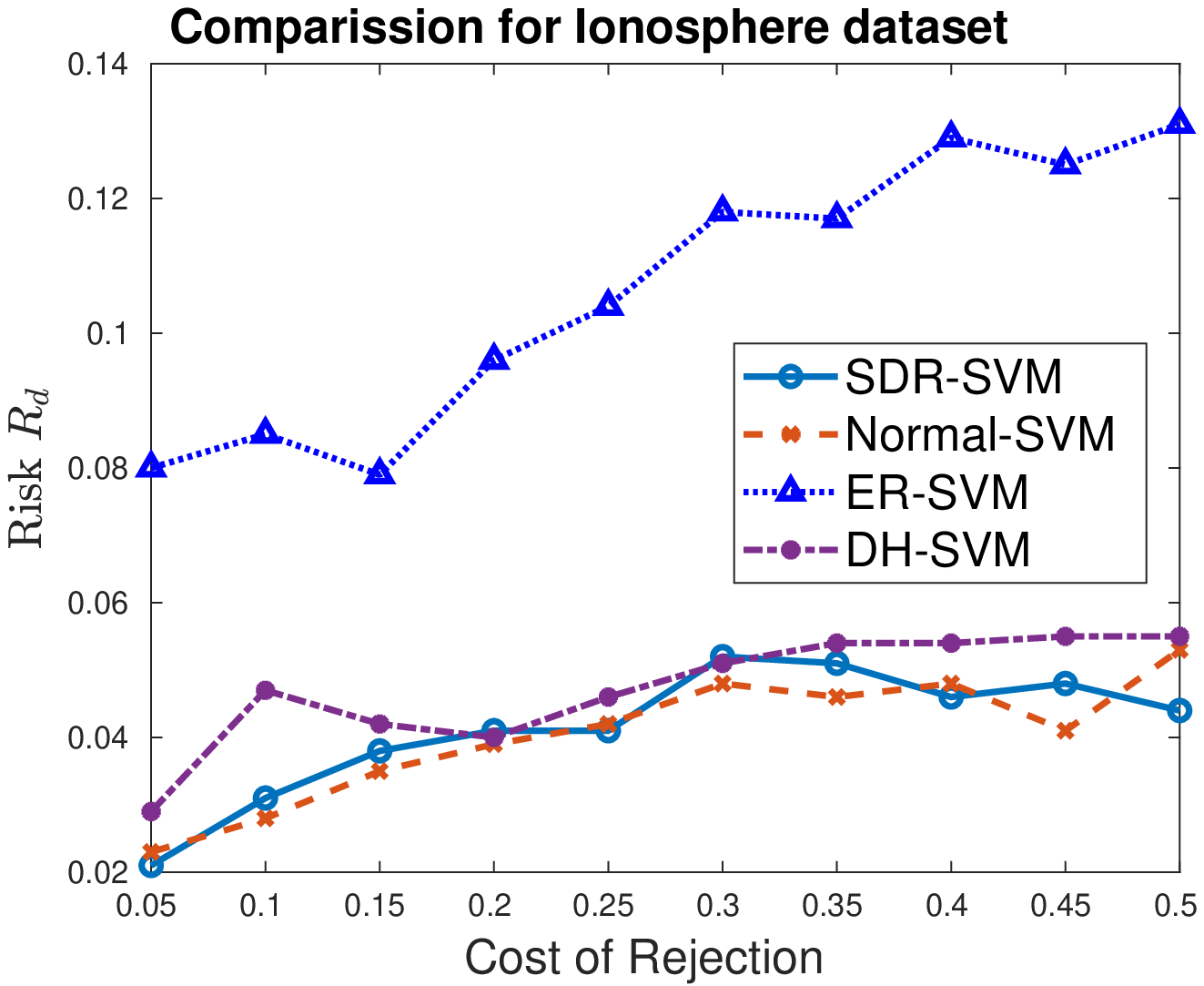}&
\includegraphics[scale=0.3]{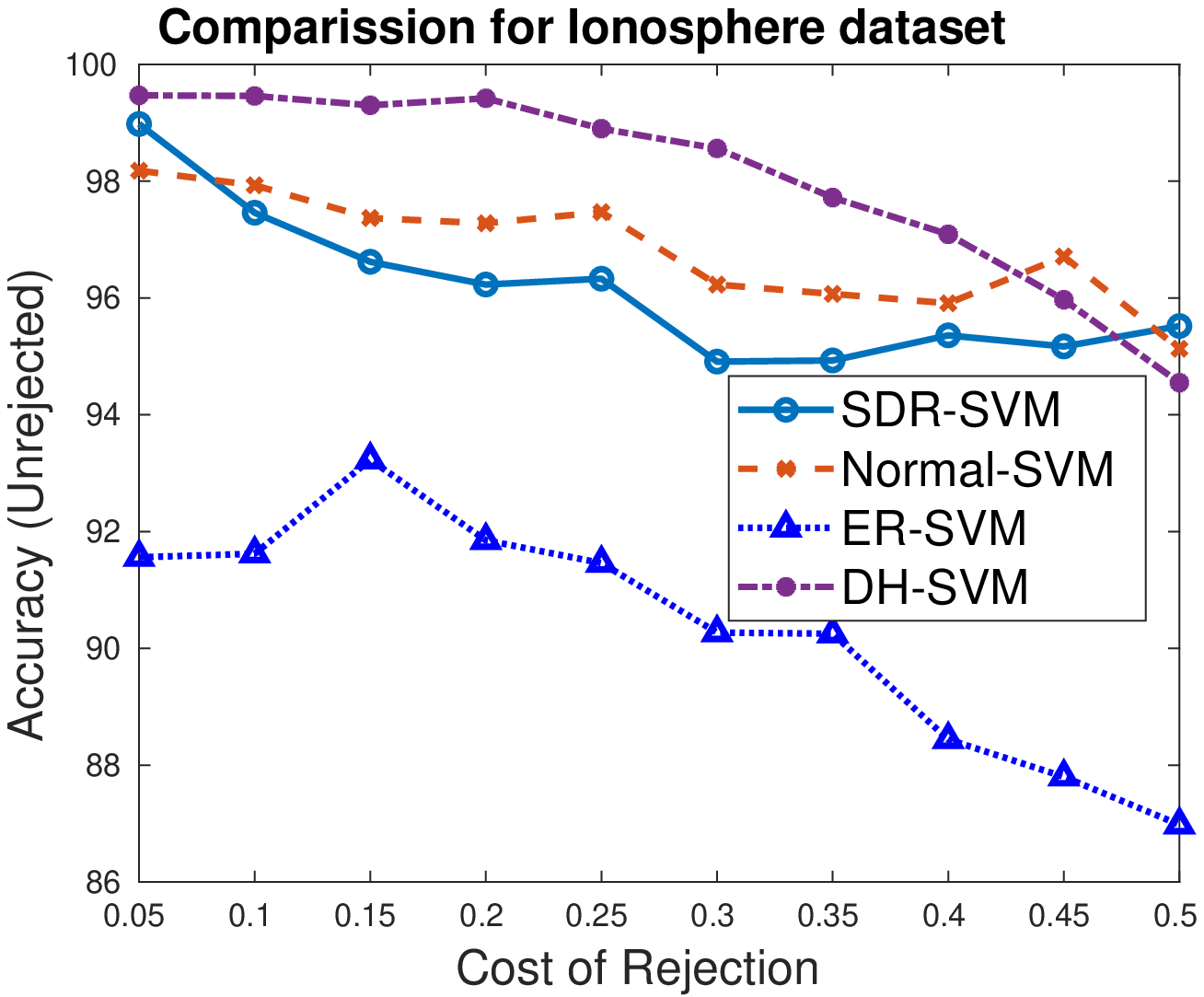} & \includegraphics[scale=0.3]{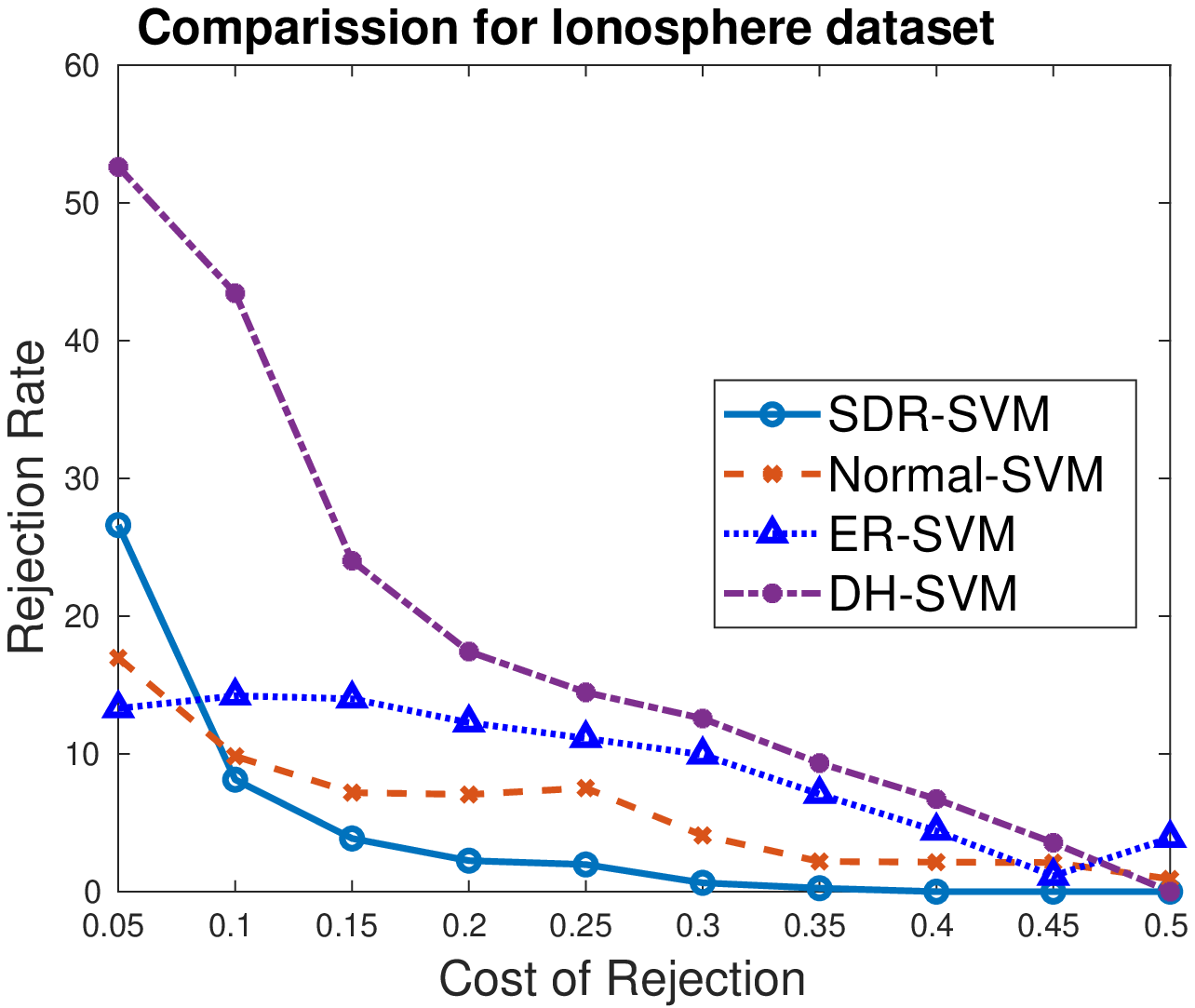}\\
\includegraphics[scale=0.3]{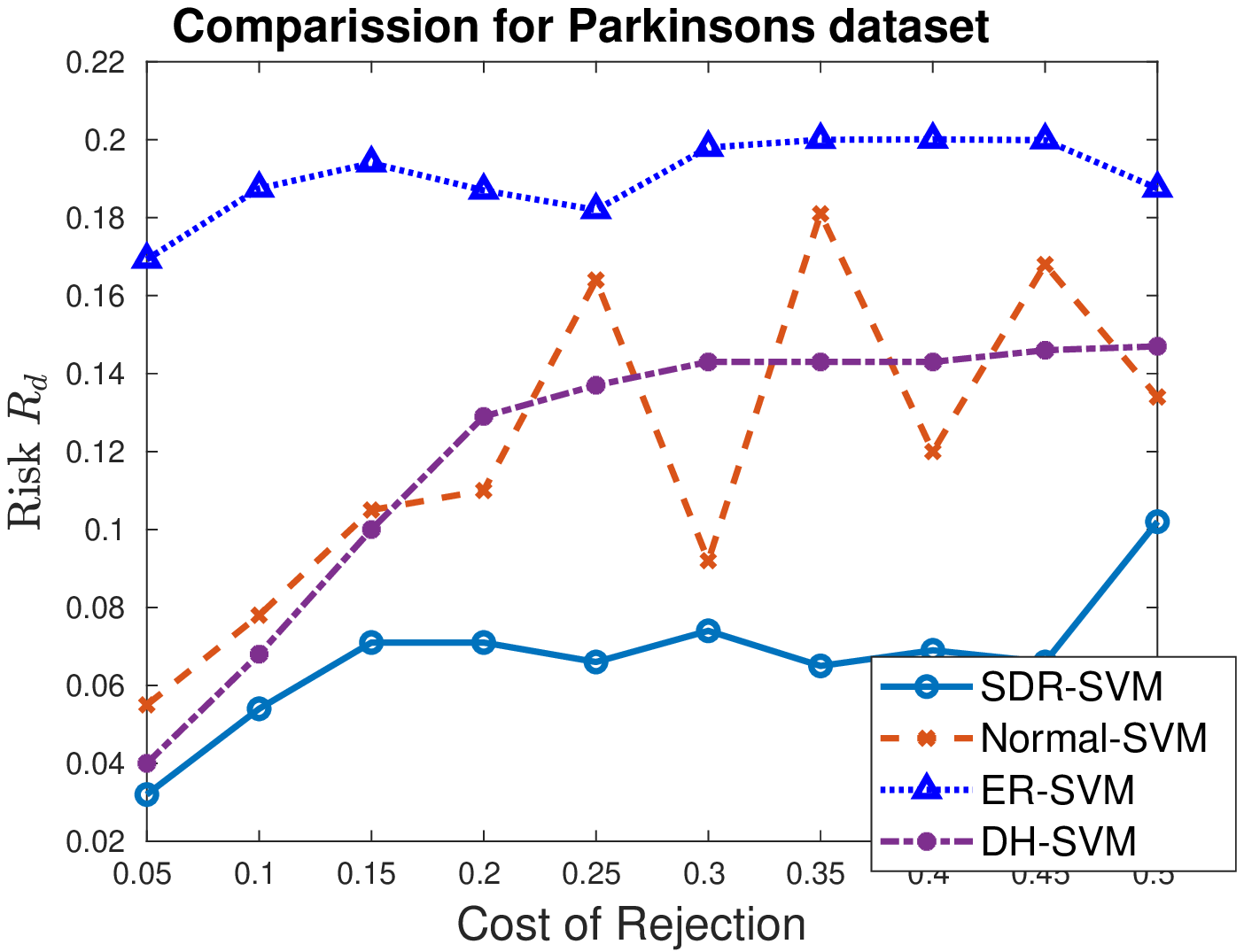}&
\includegraphics[scale=0.3]{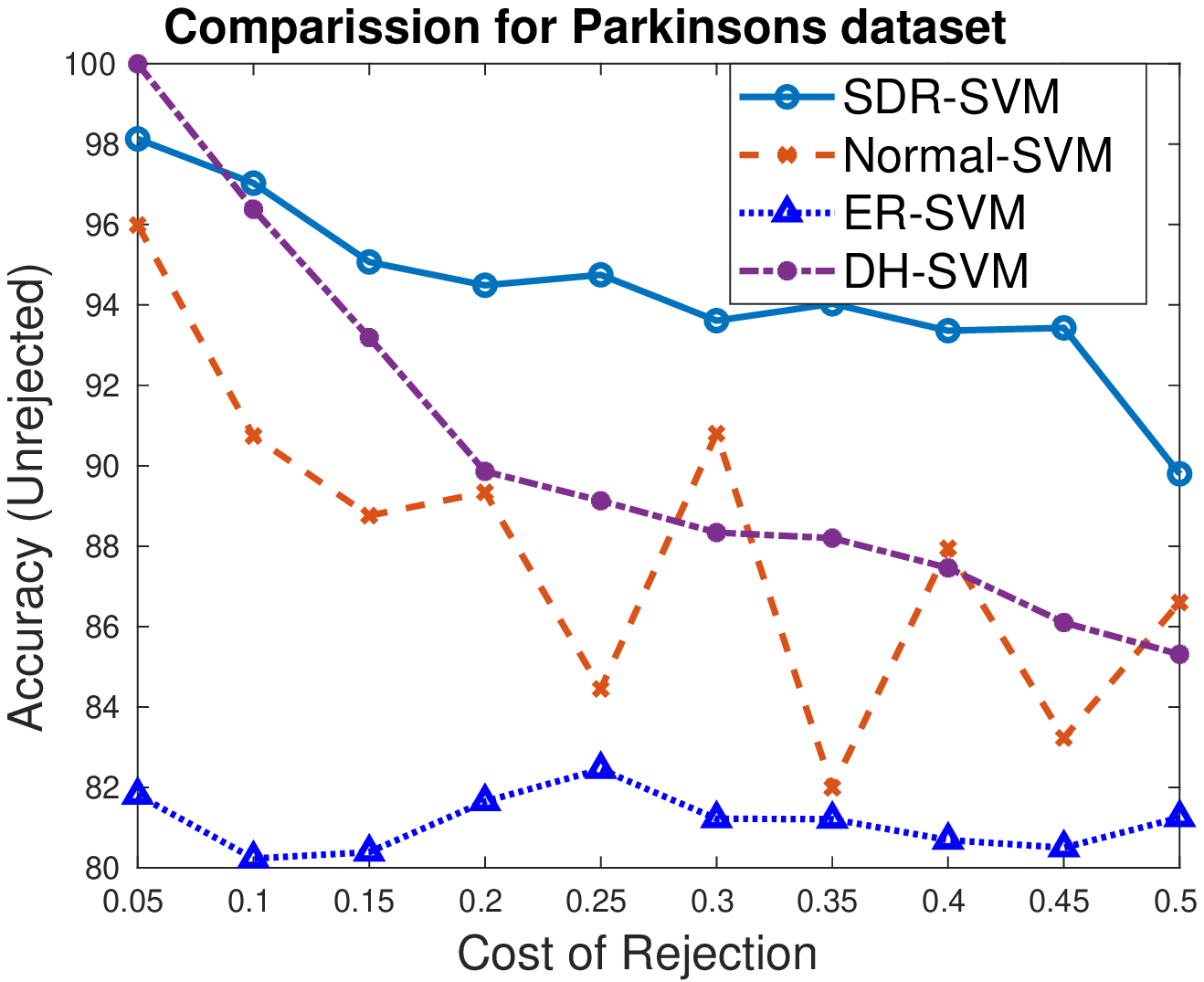} & \includegraphics[scale=0.3]{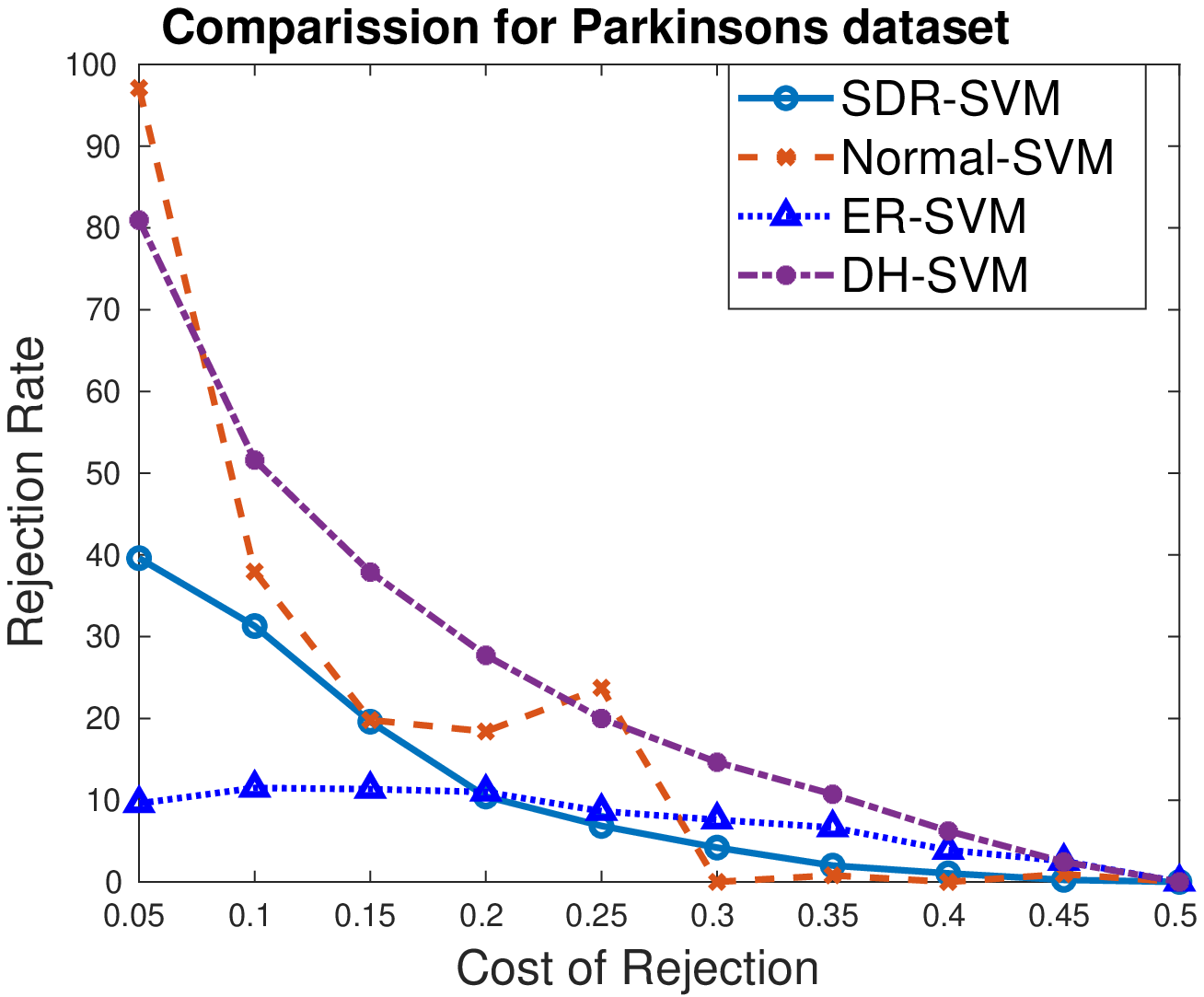} \\
\includegraphics[scale=0.3]{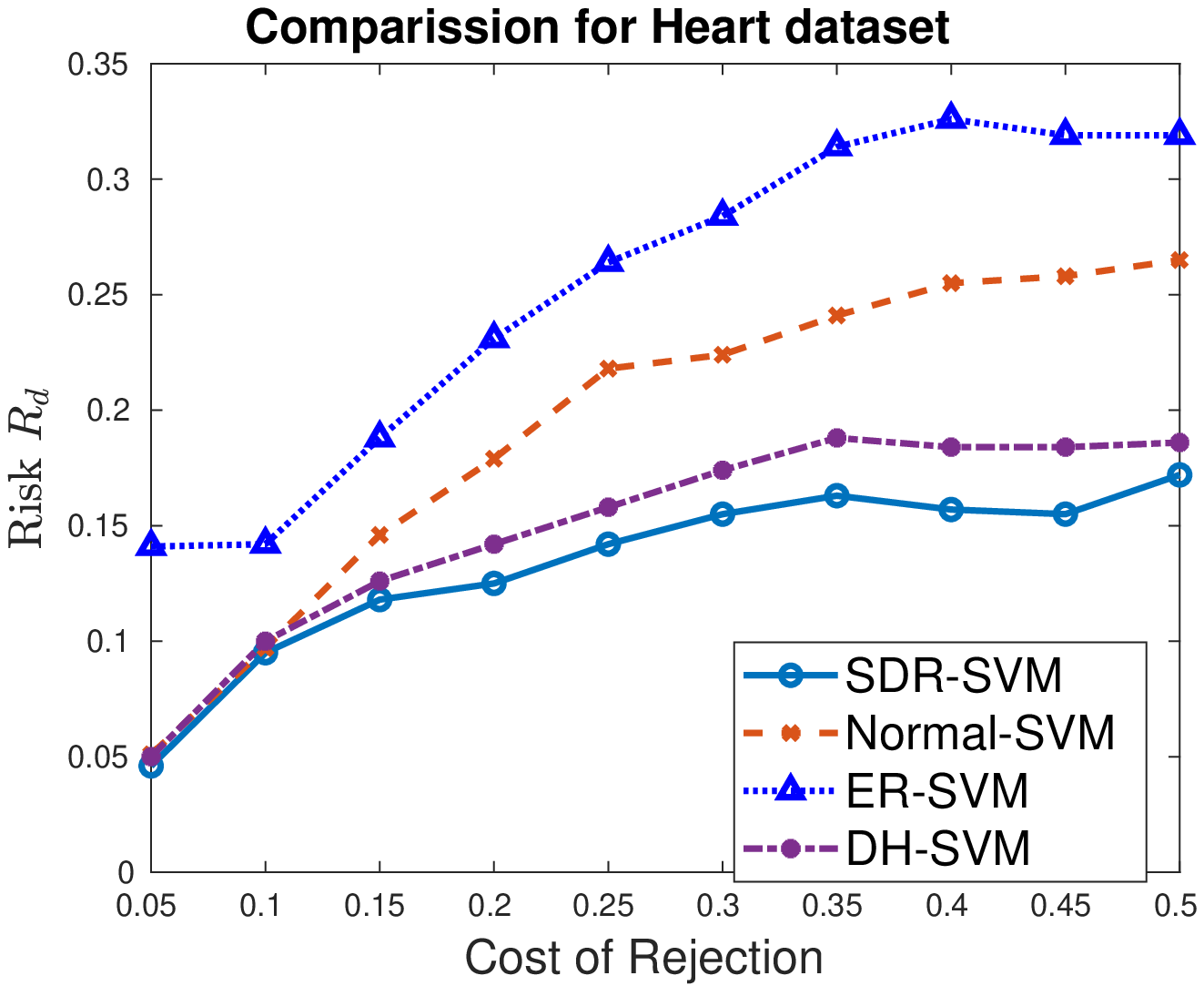}&
\includegraphics[scale=0.3]{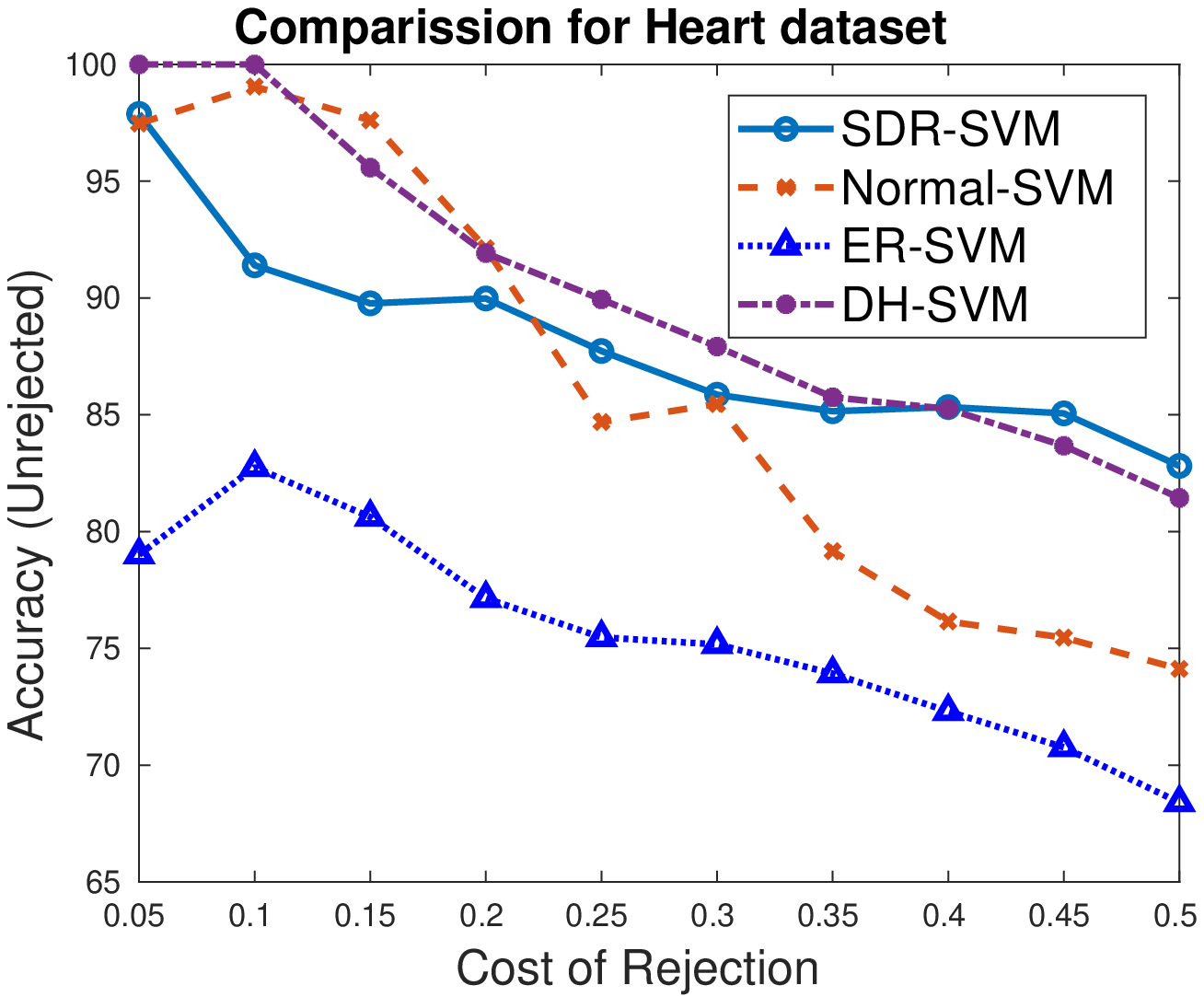} & \includegraphics[scale=0.3]{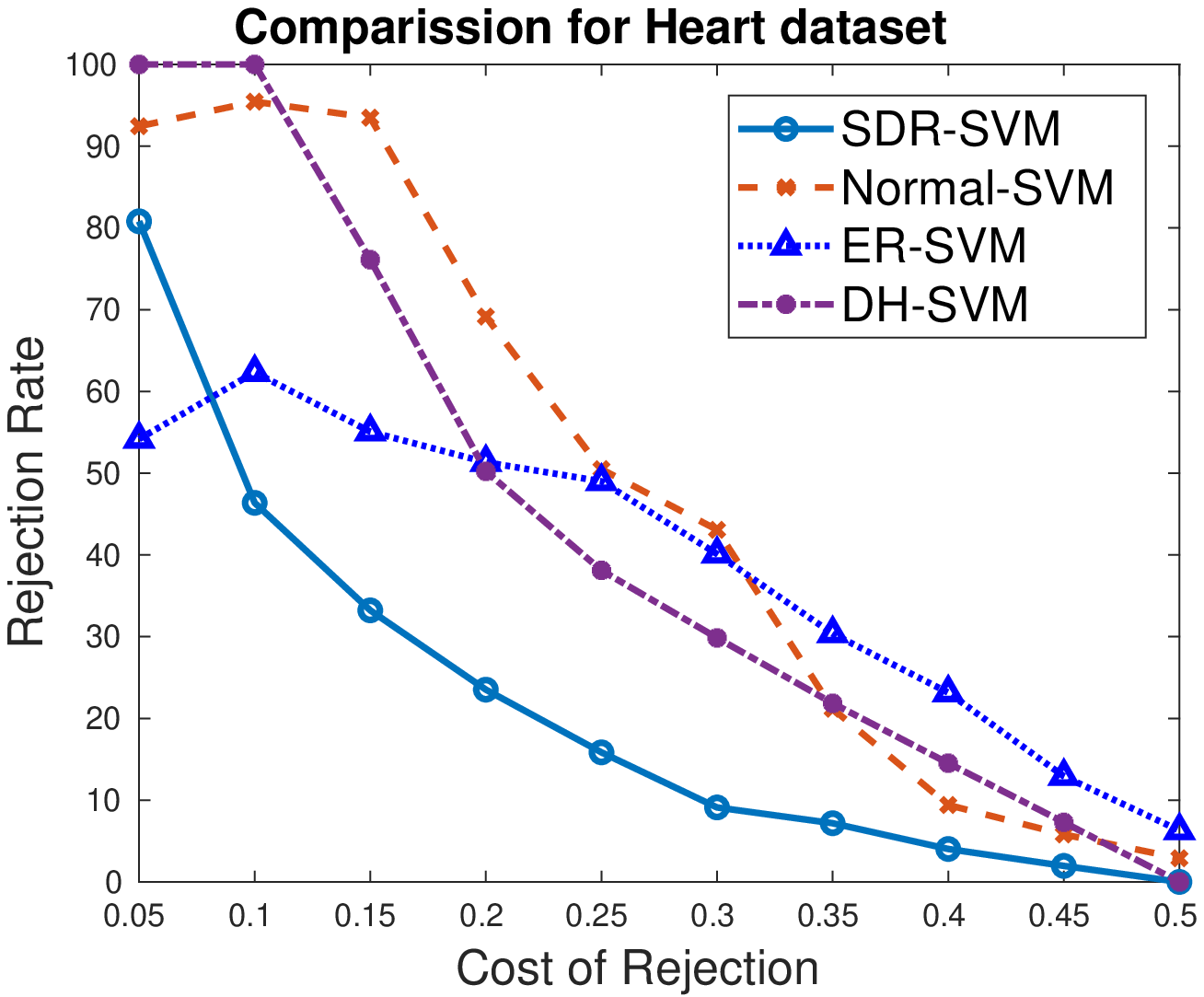} \\
\includegraphics[scale=0.3]{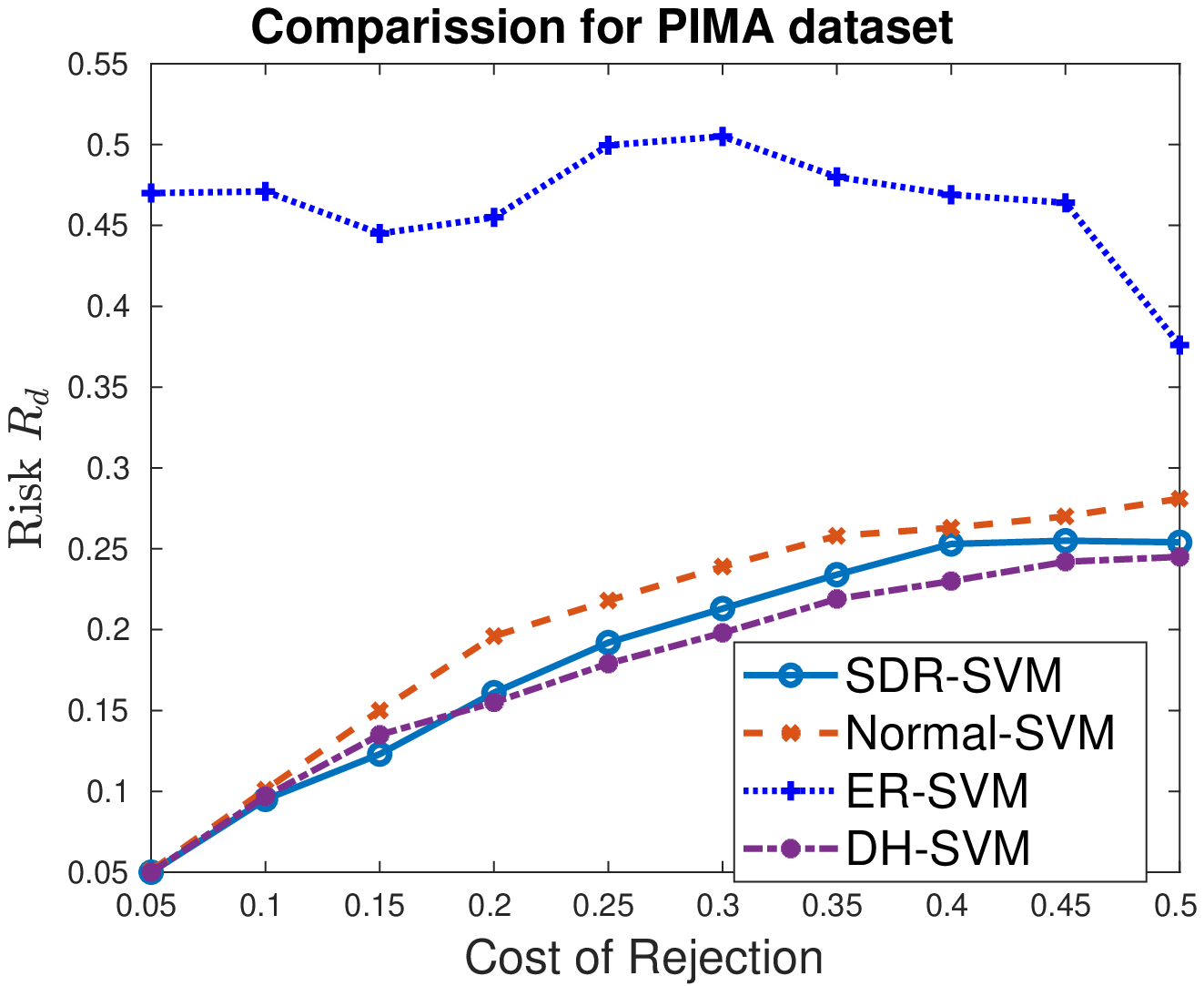}&
\includegraphics[scale=0.3]{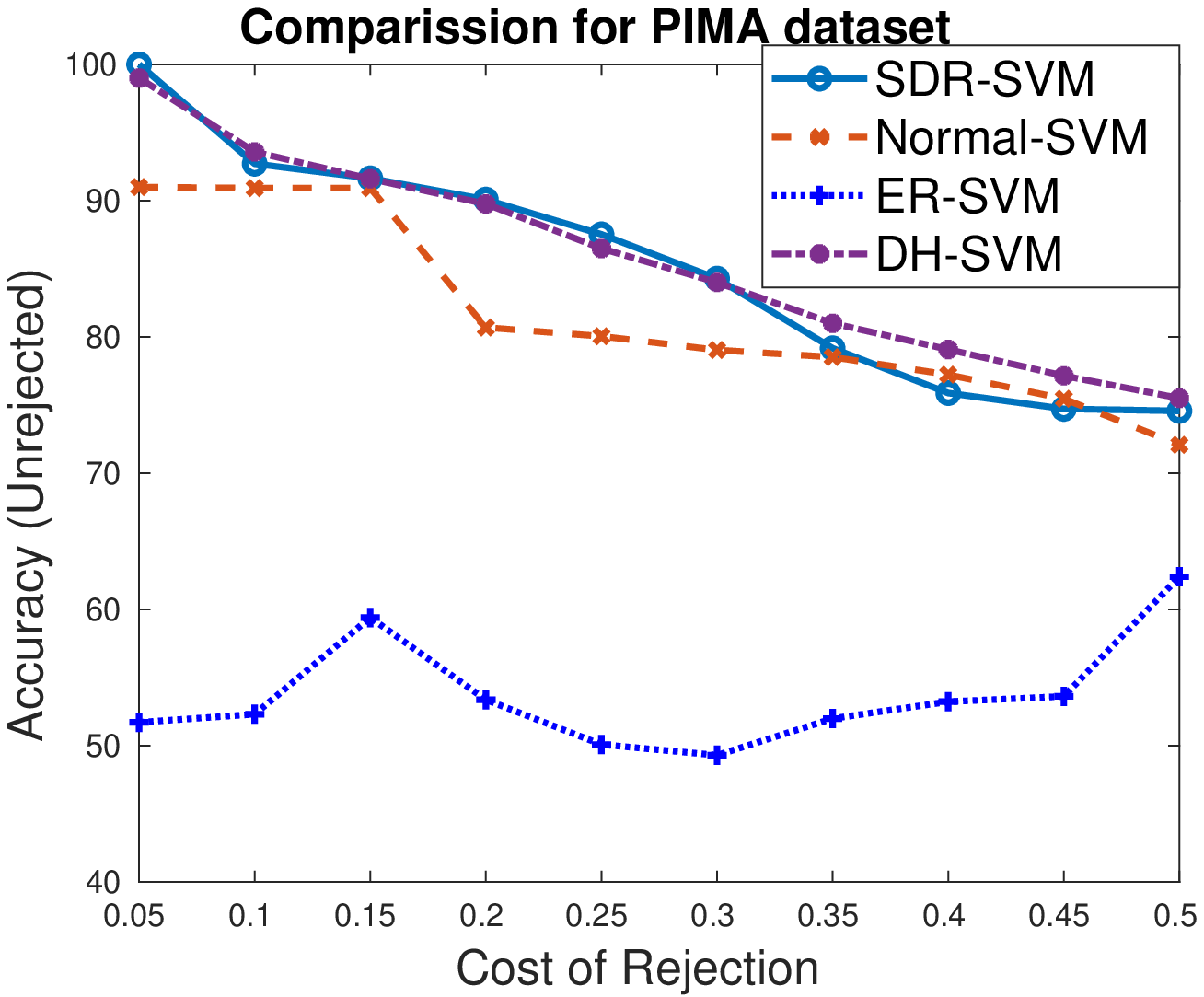} & \includegraphics[scale=0.3]{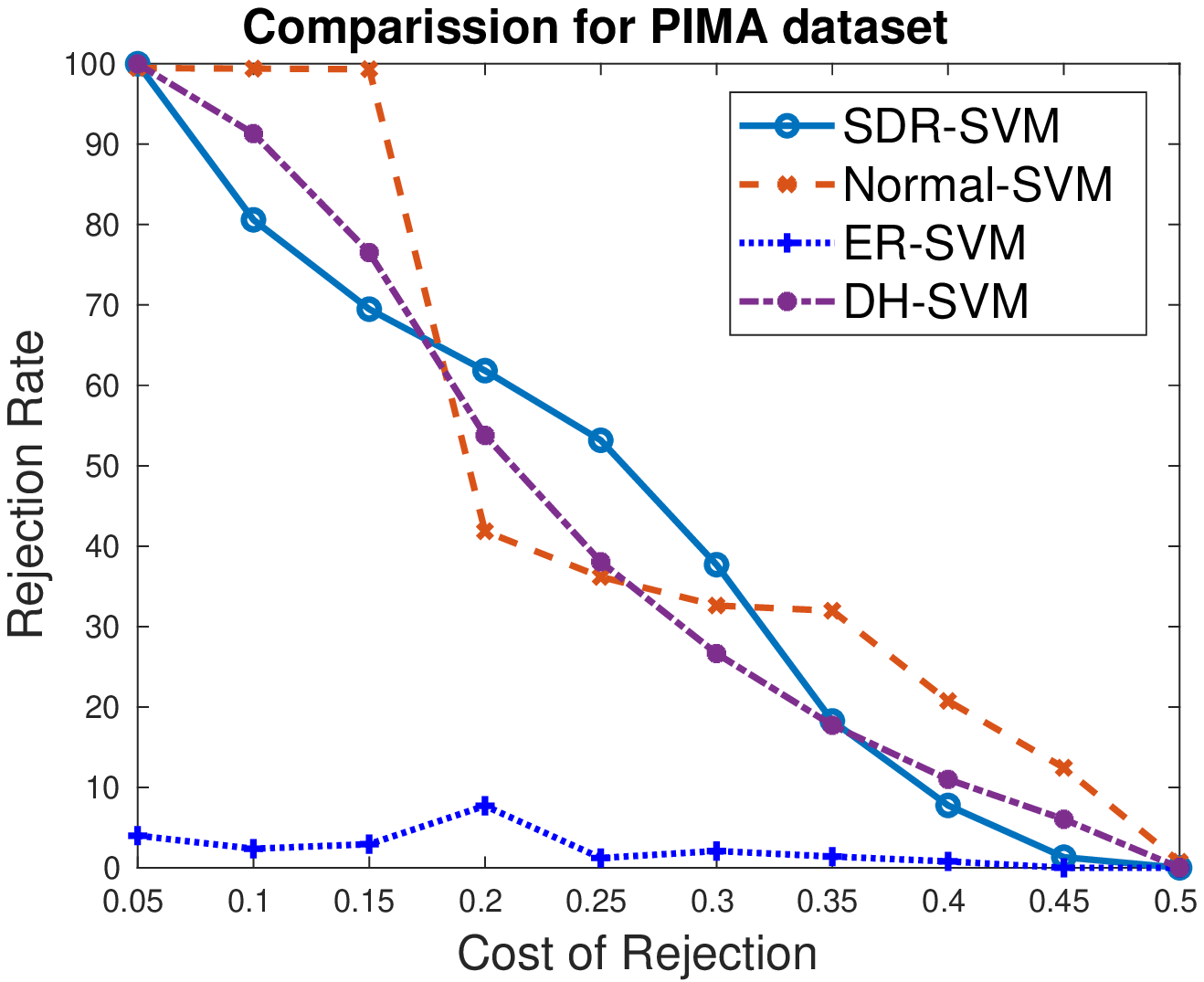} 
\end{tabular}
\caption{Comparison Plots for Different Datasets. Column 1 shows the risk $R_d$, column 2 shows accuracy on un-rejected examples, column 3 shows the rejection rate.}
\label{Fig:Comparisons}
\end{center}
\end{figure*}

\subsection{Experimental Setup}
In the proposed approach, to solve linear programming problem in each iteration, we have used CVXOPT package in python language \cite{cvxopt}. In our experiments, we apply a Gaussian kernel $\K(\xx_{i}, \xx_{j}) = \exp( -\gamma {\| \xx_{i} - \xx_{j} \|}^2)$ for nonlinear problems. In all the experiments, we set $\mu$ = 1. Regularization parameter $\lambda$ and kernel parameter $\gamma$ are chosen using 10-fold cross validation. 

We compare the performance of the proposed approach (SDR-SVM) with 3 other approaches as follows. The first approach is standard SVM based reject option classifier. In that approach, we first learn a learning decision boundary using SVM and then set the width of rejection region by cross-validation such that empirical risk under $L_{d,\rho}$ is minimized. We use this approach as a proxy for the approach proposed in Bartlett and Wegkamp \shortcite{Bartlett:2008}. Again, parameters of SVM ($C$ and $\gamma$) are learnt using 10-fold cross-validation. The second approach is the SVM with embedded reject option (ER-SVM) \cite{Fumera2002}. We used the code for this approach available online \cite{Fumera-Code}. We also compare our approach with Double hinge SVM (DH-SVM) based reject option classifier \cite{Grandvalet2008}. 

\subsection{Simulation Results}
We report the experimental results for different values of $d\in [0.05, 0.5]$ with the step size of 0.05. For every value of $d$, we find the cross-validation risk (under $L_{d,\rho}$), \% rejection rate (RR), \% accuracy on the un-rejected examples (Acc). We also report the average number of support vectors (the corresponding ${\alpha}_{i}\geq 10^{-6}$). The results provided here are based on 10 repetitions of 10-fold cross-validation (CV). 

Now we discuss the experimental results. Figure~\ref{Fig:Comparisons} shows the comparison plots for different datasets. We observe the following.

\begin{enumerate} 
\item {\bf Average Cross Validation Risk $\sR_d$: }We see that SDR-SVM performs better than ER-SVM with huge gaps in terms of the average cross validation risk ($\hat{\sR}_d$) for all datasets and for all values of $d$. For Parkinsons and Heart datasets, SDR-SVM has smaller $\hat{\sR}_d$ risk (for all values of $d$) compared to DH-SVM. For ILPD, Ionosphere and PIMA datasets, $\hat{\sR}_d$ risk of SDR-SVM is comparable to DH-SVM. SDR-SVM performs better than Normal-SVM based approach on Parkinsons, Heart, ILPD and PIMA datasets. For Ionosphere dataset, SDR-SVM performs comparable to Normal-SVM based approach. 
\item {\bf Rejection Rate: } We observe that for Inosphere, Heart and Parkinsons datasets, rejection rate of SDR-SVM is much smaller compared to other approaches except for smaller values of $d$ (0.05 and 0.1). For PIMA and ILPD datasets, the rejection rates of SDR-SVM are comparable to DH-SVM. The rejection rates for these two datasets are comparatively higher for all values of $d$. Possible reason for that could be high overlap between the two class regions. 
\item {\bf Performance on Unrejected Examples: }We see that SDR-SVM also gives good classification accuracy on unrejected examples. It always gives better accuracy compared ER-SVM. As compared to normal SVM based approach, SDR-SVM does always better on ILDP and Parkinsons datasets. For rest of the datasets, SDR-SVM gives comparable accuracy to normal SVM based method on unrejected examples. Compared to double hinge SVM, SDR-SVM does comparable to DH-SVM.
\end{enumerate}
Thus, overall SDR-SVM learns reject option classifiers which attain smaller $\hat{\sR}_{d}$ risk. It achieves this goal by simultaneously minimizing the rejection rate and mis-classification rate on unrejected examples.
\subsection{Sparseness Results}
We now show that SDR-SVM learns sparse reject option classifiers. As discussed, by sparseness we mean that the resulting classifier can be represented as a linear combination of a very small fraction of training points. 
Sparseness results for SDR-SVM are shown in Figure~\ref{fig:sparseness}.

\begin{figure}[t]
\centering
\begin{tabular}{cc}
 \includegraphics[scale=0.28]{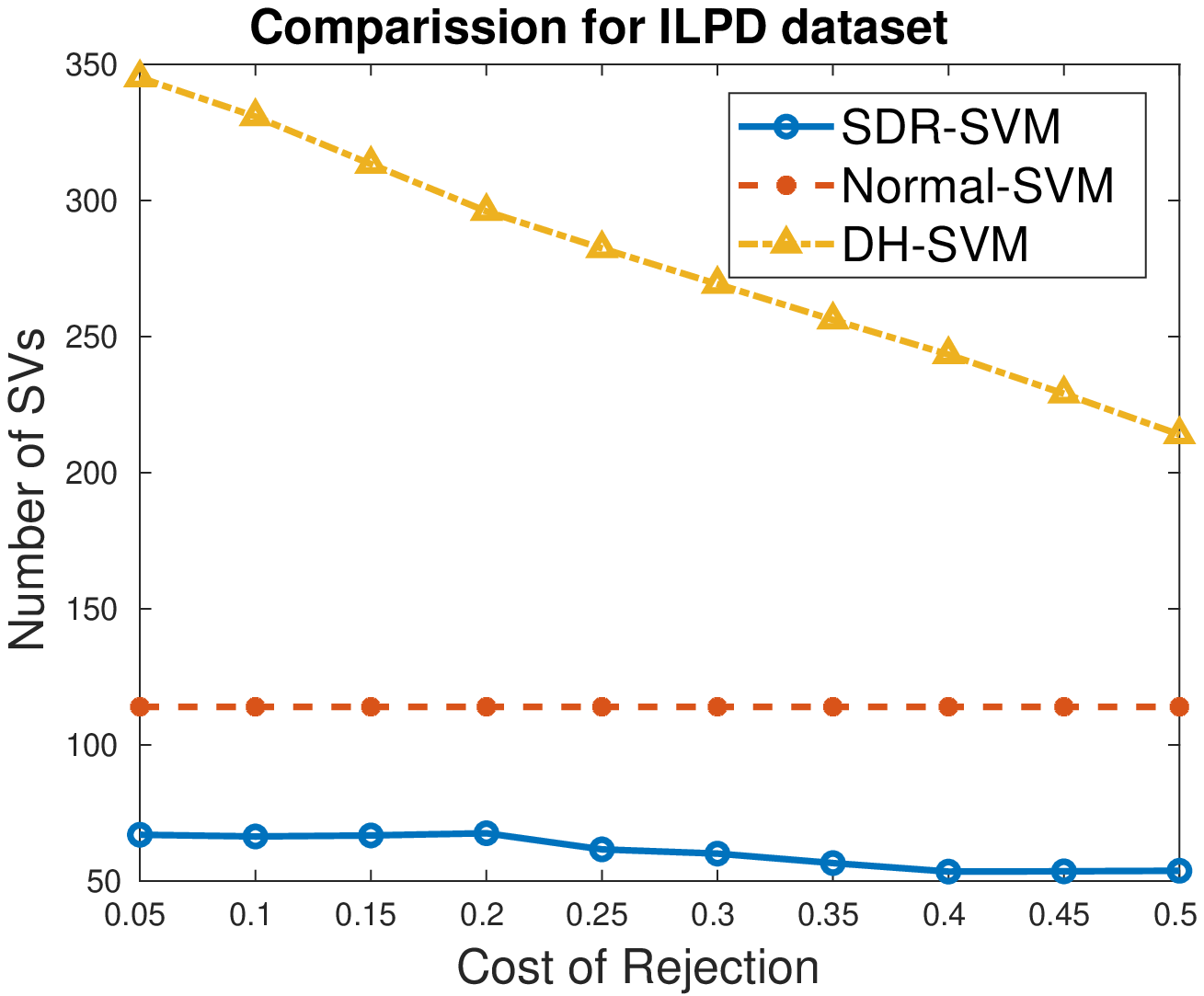} &
 \includegraphics[scale=0.28]{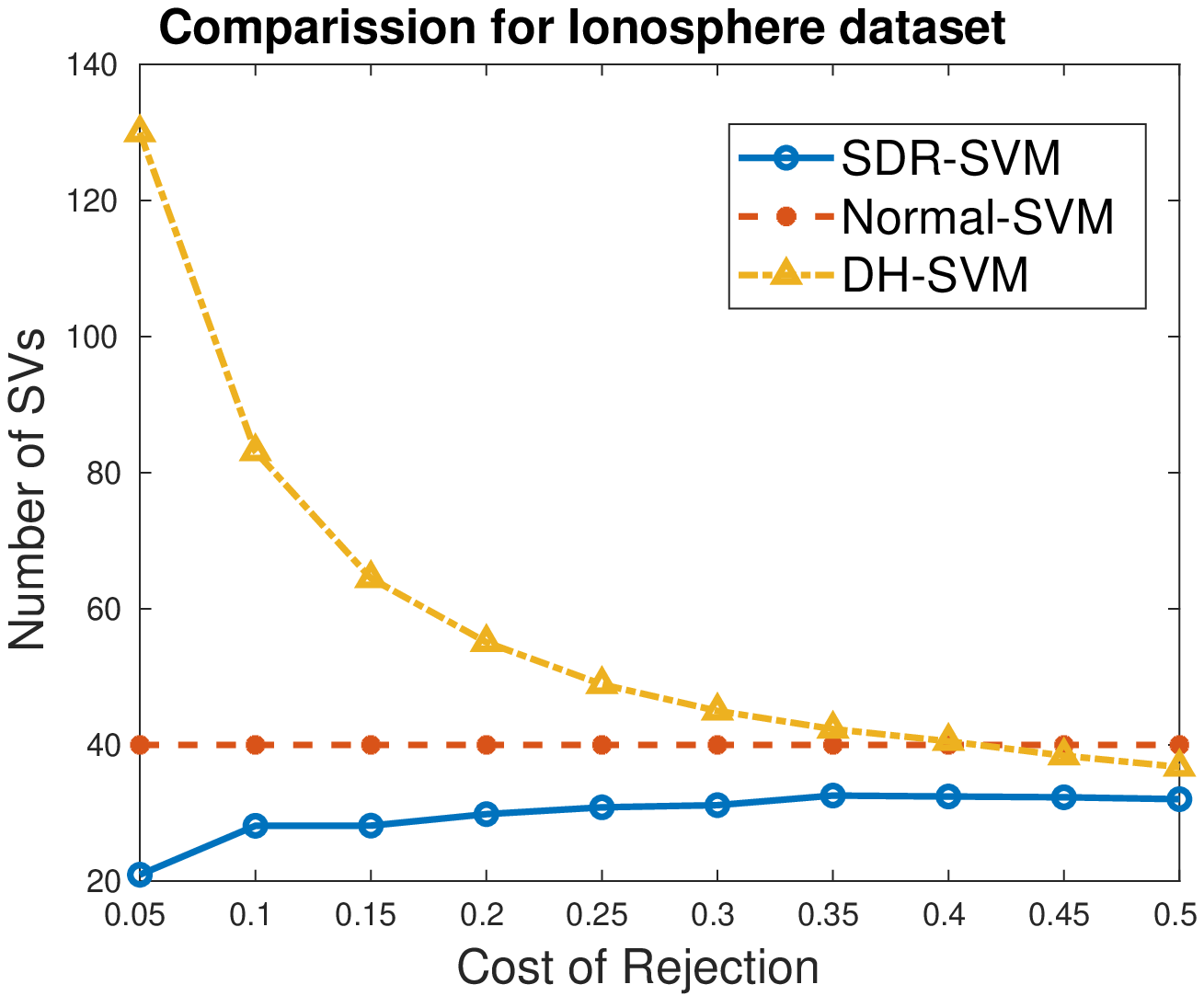} \\
 \includegraphics[scale=0.28]{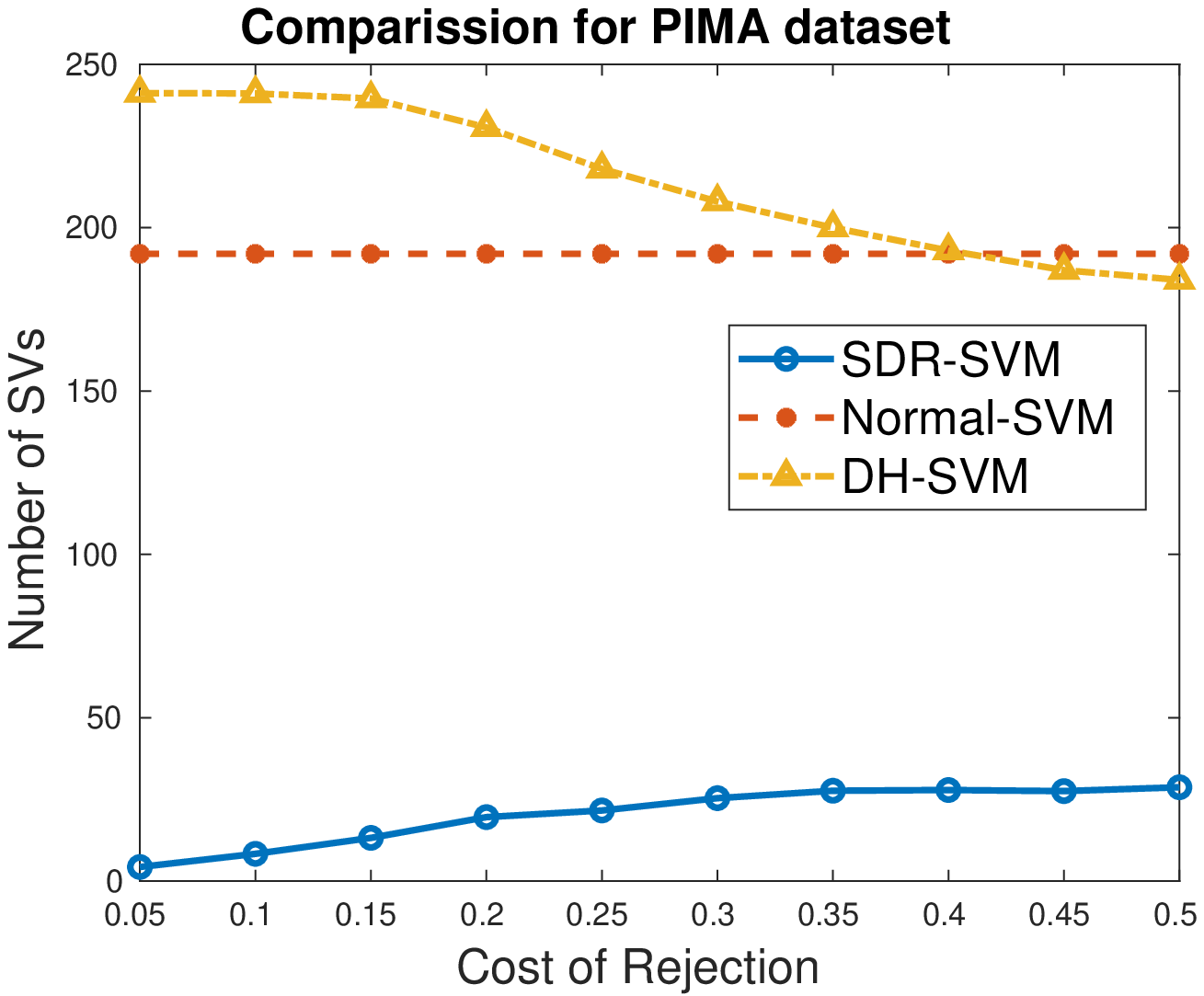}&
\end{tabular}
\caption{Sparseness Comparison of SDR-SVM with DH-SVM and Normal-SVM}
\label{fig:sparseness}
\end{figure}

We see that for ILPD, Ionosphere and PIMA datasets, SDR-SVM outputs classifiers which are much sparser compared to DH-SVM and Normal-SVM based approaches. ER-SVM does not have obvious representation for the classifier as a linear combination of training examples. 

\subsection{Experiments with Noisy Data}
$L_{dr,\rho}$ is generalization of ramp loss function for the reject option classification. For normal binary classification problem, ramp loss function is shown robust against label noise \cite{Ghosh:2015}. Motivated by the above fact, we did experiments to test the robustness of $L_{dr,\rho}$ against uniform label noise (with noise rates of $10\%, 20\%, 30\%$). Figure~\ref{fig:noise}.
\begin{figure}[t]
\begin{tabular}{cc}
\includegraphics[scale=0.28]{iono_risk.eps}&
\includegraphics[scale=0.28]{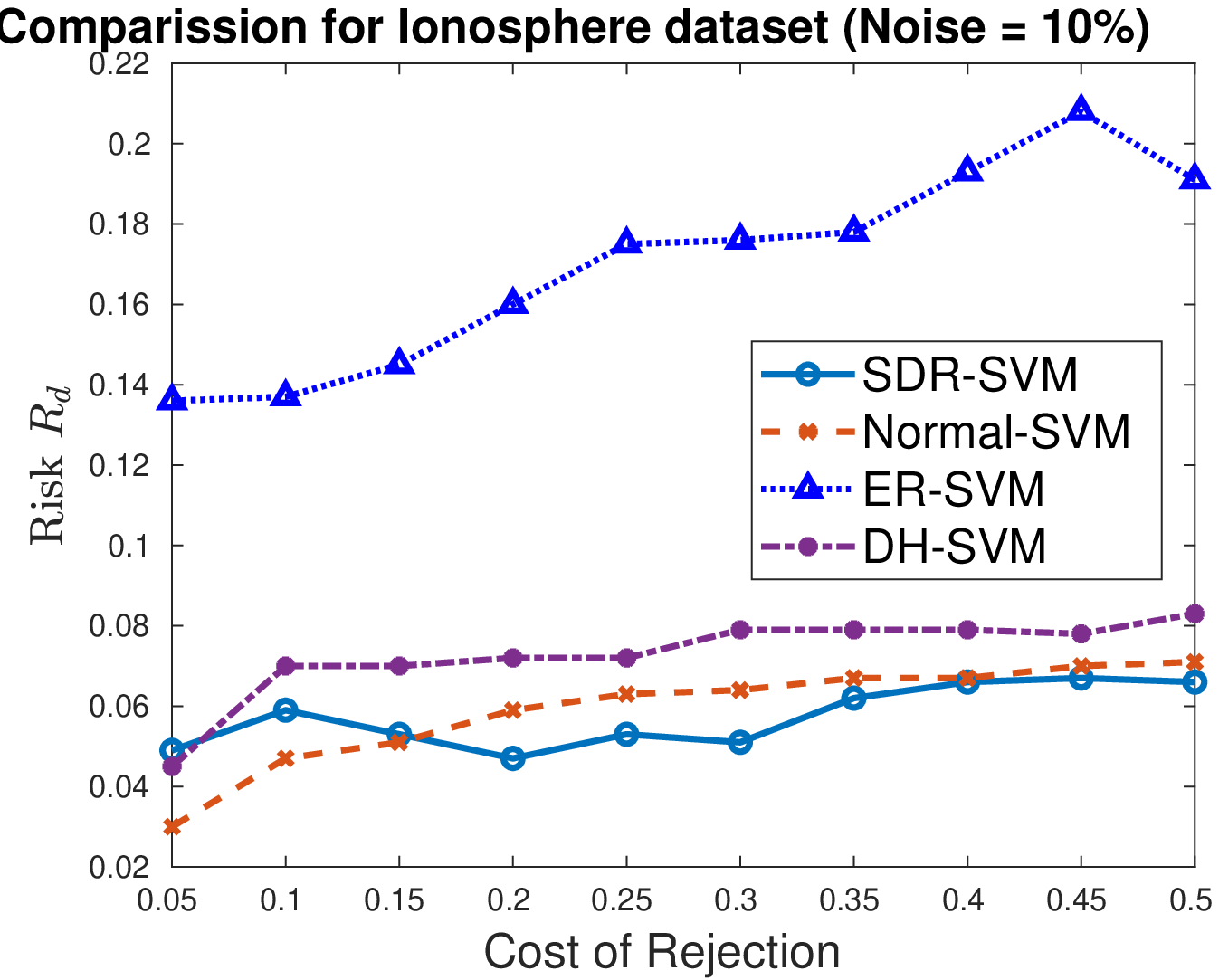}\\
\includegraphics[scale=0.27]{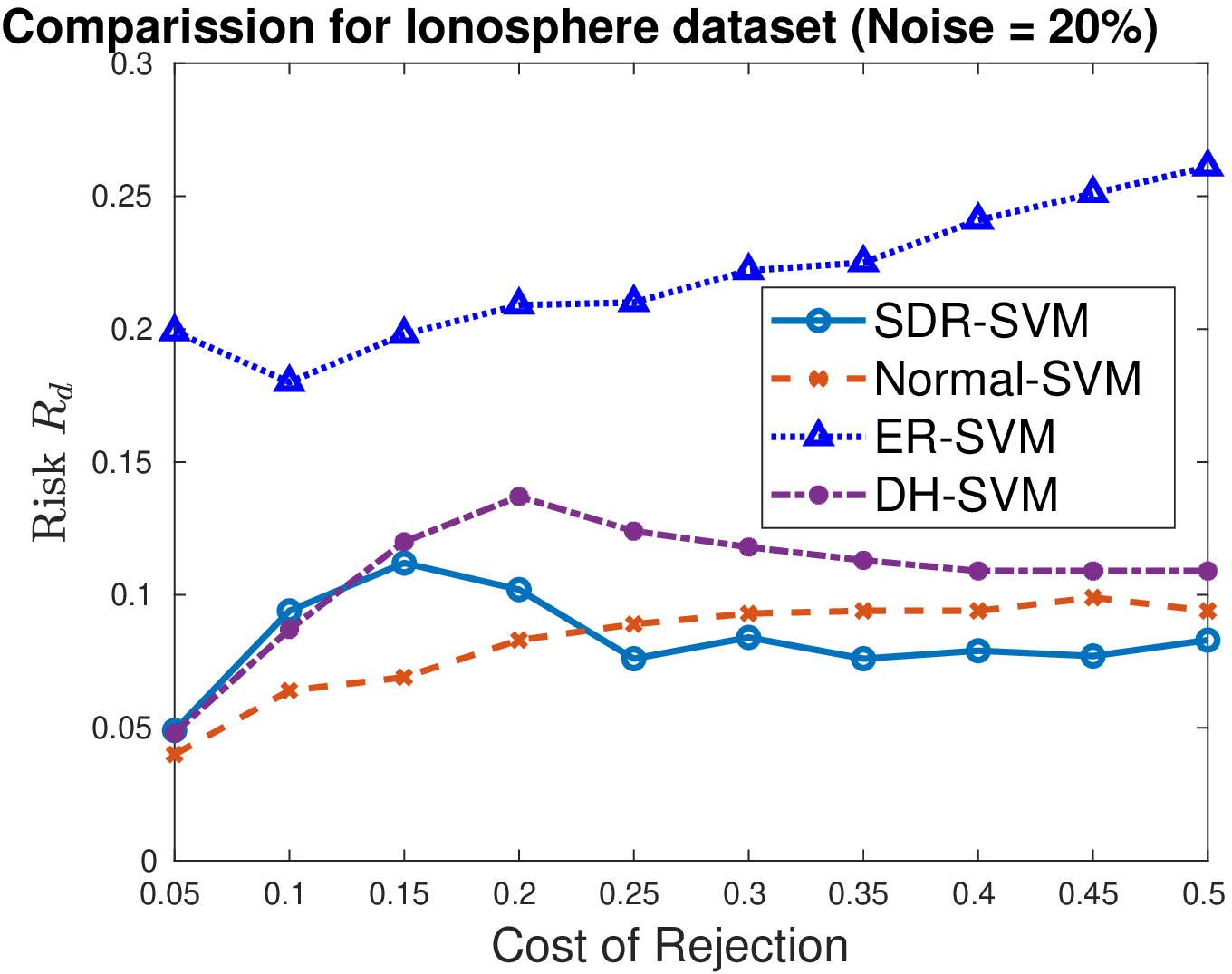}&
\includegraphics[scale=0.27]{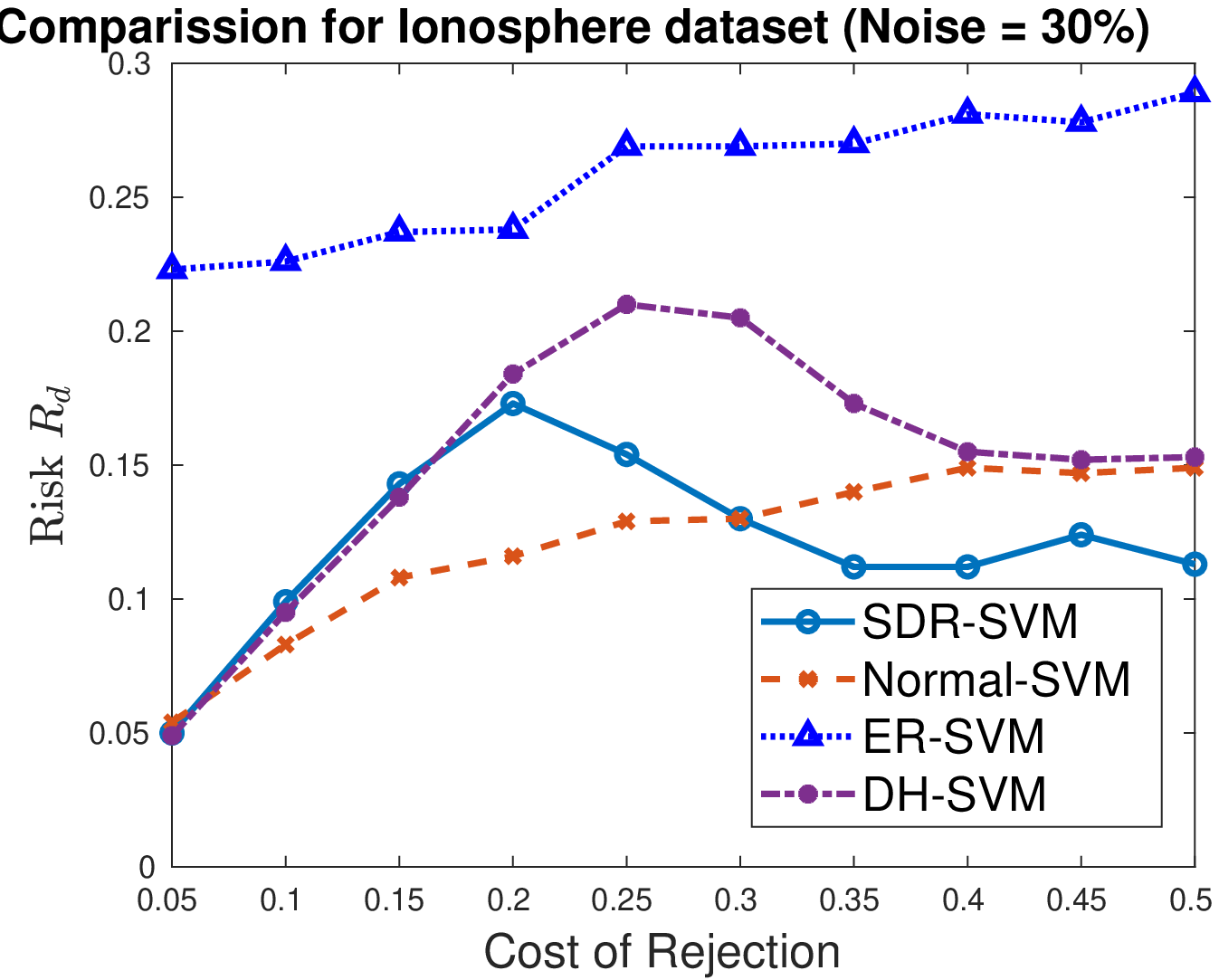}
\end{tabular}
\caption{Comparison Results in presence of uniform Label Noise}
\label{fig:noise}
\end{figure}
We observe the following.
\begin{enumerate}
\item We observe that with 10\% noise rate, increment in the risk for SDR-SVM is not significant. As we increase the noise rate, model in reject option classification confuses more for classifying the examples, therefore model tries to put more examples in rejection region for smaller values of $d$. Which leads to increase in width of rejection region. Thus, for smaller values of $d$, risk is dominated by rejection cost for proposed approach. But as we increase $d$, cost of rejection also increases and model in label noise will force examples to classify to one of the label.
With increasing noise rate, SDR-SVM remains robust for higher values of $d$.
\item Compared to ER-SVM, SDR-SVM does significantly better for all values of $d$ and for all noise rates.
\item For large values of $d$, SDR-SVM performs better than DH-SVM and noraml SVM in presence of label noise. 
\end{enumerate}

\section{Conclusions}\label{sec:conclusions}
In this paper, we proposed sparse approach for learning reject option classifier using double ramp loss. We propose a DC programming based approach for minimizing the regularized risk. The approach solves successive linear programs to learn the classifier. Our approach also learns nonlinear classifier by using appropriate kernel function. Further, we have shown the Fisher consistency of double ramp loss $L_{dr,\rho}$. We upper bound the excess risk of $L_d$ in terms of excess risk of $L_{dr}$. We then derive generalization bounds for SDR-SVM. We showed experimentally that the proposed approach does better compared to the other approaches for reject option classification and learns sparse classifiers. We also experimental evidences to show robustness of SDR-SVM against the label noise.

\bibliography{ijcai18}
\bibliographystyle{aaai}

\appendix
\section{Proof of Theorem~1} 
Generalized Bayes classifier in the context of reject option classifier is defined as follows.
\begin{equation}\label{gbd}
f_d^*(\xx) = 1.\I_{\{\eta(\xx) > 1-d\}}
+0.\I_{\{d\leq \eta(\xx) \leq 1-d\}}
-1.I_{\{\eta(\xx) < d\}}
\end{equation}

$\sR_{dr}(f,\rho)=\E_{\xx} \big[ \E_{y|\xx}[L_{dr}(yf(\xx),\rho)] \big]$. Let $r_{\eta}(f(\xx))= \E_{y|\xx}[L_{dr}(yf(\xx),\rho)]$ and $z=f(\xx)$. Thus,
$r_{\eta}(z)={\eta}L_{dr}(z,\rho) + (1-\eta) L_{dr}(-z,\rho)$. The function $ r_{\eta}(z) $ can take different values in different cases as in eq.~(\ref{cond-risk}).
\begin{table*}[t]
\begin{equation}
\label{cond-risk}
r_{\eta}(z)=
\begin{cases}
{\eta}(1+\mu)               & \text{if}\enspace z\leq  \mrmo \\
      {\eta}(1+\mu)+(1-\eta)(\mu+z+\rho){\frac{d}{\mu}}        					& \text{if}\enspace \mrmo\leq z\leq \mrms \\
      {\eta}d(1+\mu)+{\eta}(\mu-z-\rho){\frac{(1-d)}{\mu}}+(1-\eta)(\mu+z+\rho){\frac{d}{\mu}}& \text{if}\enspace \mrms \leq z\leq \mrps \\
      {\eta}d(1+\mu)+{\eta}(\mu-z-\rho){\frac{(1-d)}{\mu}}+(1-\eta)(1+\mu)d     & \text{if}\enspace \mrps\leq z\leq \mrpo \\
      {\eta}d(1+\mu)+(1-\eta)(1+\mu)d = d(1+\mu)    									& \text{if}\enspace \mrpo\leq z\leq \prmo \\
      {\eta}d(1+\mu)+(1-\eta)(1+\mu)d+(1-\eta)(z\mrpo)\frac{(1-d)}{\mu}       & \text{if}\enspace \prmo\leq z\leq \prms \\
	  {\eta}(\prpo-z)\frac{d}{\mu}+(1-\eta)(1+\mu)d+(1-\eta)(z\mrpo)\frac{(1-d)}{\mu}& \text{if}\enspace \prms\leq z\leq\prps \\ 
      {\eta}(\prpo-z)\frac{d}{\mu}+(1-\eta)(1+\mu)					& \text{if}\enspace \prps\leq z\leq \prpo \\
      (1-\eta)(1+\mu)					& \text{if}\enspace z\geq \prpo 
\end{cases} 
\end{equation}
\end{table*}
From eq~(\ref{cond-risk}), we can say that 
\begin{equation*}
\frac{\partial r_{\eta}(z)}{\partial z} = 
\begin{cases}
	  =0                 & \text{if}\enspace z\leq  \mrmo \\
      > 0         					& \text{if}\enspace \mrmo\leq z\leq \mrms \\
      >0\text{ or }<0 & \text{if}\enspace \mrms \leq z\leq \mrps \\
      <0     & \text{if}\enspace \mrps\leq z\leq \mrpo \\
      =0							& \text{if}\enspace \mrpo\leq z\leq \prmo \\
      >0       & \text{if}\enspace \prmo\leq z\leq \prms \\
	  >0\text{ or }<0				& \text{if}\enspace \prms\leq z\leq\prps \\ 
      <0		& \text{if}\enspace \prps\leq z\leq \prpo \\
      =0					& \text{if}\enspace z\geq \prpo \\
\end{cases}
\end{equation*}
Minimum of any graph will occur when the slope of the graph is equal to zero and goes from negative to positive  or it can happen at any boundary point. Thus, minimum of $r_{\eta}(z)$ can occur in only following 3 intervals: $(z\leq  \mrmo)$, $(\mrpo\leq z\leq \prmo),\enspace (z\geq \prpo)$. Thus,   
\begin{equation*}
\begin{aligned}
r^{*}_{\eta}(z)
=&\min \big(r_{\eta}(z)\I_{\{ z\leq  \mrmo \}},r_{\eta}(z)\I_{\{ |z|\leq \prmo \}},r_{\eta}(z)\I_{\{ z\geq \prpo \}} \big) \\
=& \min\big( {\eta}(1+\mu),\enspace d(1+\mu),\enspace (1-\eta)(1+\mu)\big)
\end{aligned}
\end{equation*}
Now, if $\eta < d$, then $\eta(1+\mu) < d(1+\mu)$.We know that $d \leq 0.5$ therefore if $\eta < d$ then $\eta < 0.5$. If $\eta < 0.5$ then $\eta < 1 - \eta$ which implies that $\eta(1 + \mu) < (1 - \eta)( 1 + \mu ) $ therefore for $\eta < d$, $r^{*}_{\eta}(z) = \eta (1+\mu)$. If $\second$ then $d \leq \eta$ which implies that $d(1 + \mu) \leq \eta( 1 + \mu )$. If $\eta \leq (1-d)$ then $d \leq (1 - \eta)$ which implies $d(1 + \mu) \leq (1-\eta)(1+\mu)$ therefore for $\second$, $r^{*}_{\eta}(z) = d (1+\mu)$. If $\third$ then $d > (1 - \eta)$ which implies that $d(1+\mu) > (1-\eta)(1+\mu)$. We know that $d \leq 0.5$ therefore $\eta > 0.5$ which implies that $\eta > (1 - \eta)$ and $\eta(1 + \mu) > (1-\eta)(1+\mu)$ therefore for $\third$, $r^{*}_{\eta}(z) = (1 - \eta) (1+\mu)$. Combining above statements, 
\begin{equation}
r^{*}_{\eta}(f)=
\begin{cases}
	{\eta}(1+\mu)       	& \text{if}\enspace \first \\
    d(1+\mu)				& \text{if}\enspace \second \\
    (1-\eta)(1+\mu)			& \text{if}\enspace \third \\
\end{cases}
\end{equation}
and Bayes discriminant function for double ramp loss will be
\begin{equation}
f^{*}_{dr}(\xx)=
\begin{cases}
	-1  		& \text{if}\enspace \first \\
    0   		& \text{if}\enspace \second \\
    1			& \text{if}\enspace \third \\
\end{cases}
\end{equation}
which is same as $f^{*}_d(\xx)$. Therefore, $f_{d}^{*}(\xx)$  minimizes the risk $\sR_{dr}$.

\section{Proof of Proposition~2}
{\bf Part 1: } ${\xi}_{-1}(\eta) \leq {H_{-1}}(\eta) - H(\eta)$
\begin{equation*}
\begin{aligned}
H_{-1}(\eta) &= \inf_{z<-\rho}\;\;r_{\eta}(z)\\
&= \inf_{z<-\rho}\;\;{\eta}L_{dr}(z,\rho) + (1-\eta)L_{dr}(-z,\rho)
\end{aligned}
\end{equation*}
It can easily be seen that the graph of $r_{\eta}(z)$ is piece-wise linear. Therefore, infimum of $r_{\eta}(z)$ will occur only at the corners of graph. Thus, comparing the slopes of different linear functions, we get 
\begin{equation*}
\begin{aligned}
&H_{-1}(\eta) = f(-\rho - \mu){\I_{\{ \first \}}} +\min\big( f(-\rho), f(-\rho - \mu)\big) {\I_{\{ \eta \geq d \}}} \\
&= \eta (1 + \mu)\I_{\{ \first \}} + \min \big( \eta(1+\mu), {\eta}d(\mu - 1) + \eta + d \big)\I_{\{ \eta \geq d \}}
\end{aligned}
\end{equation*}
Also,
\begin{align*}
& \min \big( \eta(1+\mu), {\eta}d(\mu - 1) + \eta + d \big)\\
&= 
\begin{cases}
	\eta(1 + \mu)     & \text{if} \enspace \conoo	\\
    {\eta}d(\mu - 1) + \eta + d  & \text{if} \enspace \conos	
\end{cases}
\end{align*}

We now analyze different cases as follows.
\begin{enumerate}
\item $\first$:  ${\xi}_{-1}(\eta) - ({H_{-1}}(\eta) - H(\eta)) = 0$
\item $ \eta \geq d $ and $\mu \in (0,\frac{d(1 - \eta)}{\eta(1-d)})$: 
\begin{align*}
H_{-1}(\eta) - H(\eta)& = (\eta - d)(1 + \mu)\I_{\{ \second  \}} \\
&+ (2{\eta} - 1)(1 + \mu)\I_{\{ \third \}}
\end{align*}
We know that 
\begin{align*}
{\xi}_{-1}(\eta) &= \eta - {\xi}(\eta)\\
&= (\eta-d){\I_{\{ \second \}}}+ (2\eta-1){\I_{\{ \third \}}}
\end{align*}

Now we can easily see that 
\begin{equation*}
{\xi}_{-1}(\eta) - ({H_{-1}}(\eta) - H(\eta)) \leq 0
\end{equation*}
\item $ \eta \geq d $ and $\mu \in [\frac{d(1 - \eta)}{\eta(1-d)}, 1]$:
\begin{align*}
&H_{-1}(\eta) - H(\eta) = ({\eta}d(\mu - 1) + \eta -d{\mu}){\I_{\{ \second \}}} \\
&+ ({\eta}d(\mu - 1) + \eta + d - (1-\eta)(1+\mu)){\I_{\{ \third \}}}
\end{align*}
Now, 
\begin{align*}
&{\xi}_{-1}(\eta) - ({H_{-1}}(\eta) - H(\eta)) 
=d(1-\mu)(\eta-1)\I_{\{ \second \}} \\
&\;\;\;\;\;\;\;+ \big( (\eta-1)d + \mu(1-\eta-{\eta}d) \big) \I_{\{ \third \}} \\
&\leq  d(1-\mu)(\eta-1)\I_{\{ \second \}} + \big( (\eta-1)d \\
&\;\;\;\;\;\;\;+ \mu(d-{\eta}d)\big) \I_{\{ \third \}} \\
&\leq d(1-\mu)(\eta-1)\I_{\{ \second \}} +  d(1-\mu)(\eta-1)  \I_{\{ \third \}}\\
&\leq 0
\end{align*}
\end{enumerate}
Thus, $\forall \mu \in (0,1] \enspace \text{and } \forall \eta \in [0,1] $, we get,
\begin{equation*}
{\xi}_{-1}(\eta) \leq {H_{-1}}(\eta) - H(\eta)
\end{equation*}
{\bf Part 2: } ${\xi}_{r}(\eta) \leq {H_{r}}(\eta) - H(\eta)$ 
\begin{equation*}
\begin{aligned}
H_{r}(\eta) = \inf_{|z|\leq\rho}\;\;{\eta}L_{dr}(z,\rho) + (1-\eta)L_{dr}(-z,\rho)  
\end{aligned}
\end{equation*}
We use the piece-wise linear property of $r_{\eta}(z)$. $H_{r}(\eta)$ can be written as
\begin{equation*}
\begin{aligned}
H_{r}(\eta) &= \min \left\{ \eta d(\mu-1)+\eta+d, d(1+\mu)\right\} \I_{\{ \first \}}\\
&\;\;\;\;+ d(1 + \mu)\I_{\{ \second \}} \\
&\;\;\;\; + \min\left\lbrace d(1+\mu), (d{\mu}+1)(1-\eta) + {\eta}d\right\rbrace \I_{\{ \third \}}
\end{aligned}
\end{equation*}
For further analysis, we can divide $H_{r}(\eta)$ in two parts with respect to values of $\mu$ where minimum function changes value.
\begin{enumerate}
\item $\first$ and $\mu \in (0, \frac{\eta(1-d)}{d(1 - \eta)})$:
\begin{equation*}
\begin{aligned}
\xi_{r}(\eta) - (H_{r}(\eta) - H(\eta)) &= d - \eta -  ({\eta}d(\mu-1)+\eta+d)\\
&\;\;\;\;+ \eta(1+\mu)  \\
&= (1-d)\eta(\mu-1) \leq 0
\end{aligned}
\end{equation*}
\item $\first$ and $\mu \in [\frac{\eta(1-d)}{d(1 - \eta)},1]$:
\begin{equation*}
\begin{aligned}
\xi_{r}(\eta) - (H_{r}(\eta) - H(\eta))& = (d - \eta) - \big( d(1+\mu) - \eta(1+\mu)\big) \\
&= -\mu(d - \eta) \leq 0
\end{aligned}
\end{equation*}
\item $\second$:
\begin{equation*}
\begin{aligned}
\xi_{r}(\eta) - (H_{r}(\eta) - H(\eta)) &= 0 - (d(+\mu) - d(1+\mu))\leq 0
\end{aligned}
\end{equation*}
\item $\third$ and $\mu \in (0, \frac{(1-\eta)(1-d)}{{\eta}\mu})$:
\begin{equation*}
\begin{aligned}
\xi_{r}(\eta) - (H_{r}(\eta) - H(\eta)) &= (\eta + d -1) -  d(1+\mu) \\
&\;\;\;+ (1-\eta)(1+\mu) \\
&= -\mu(\eta + d -1)\\
&\leq 0
\end{aligned}
\end{equation*}
\item $\third$ and $\mu \in [\frac{(1-\eta)(1-d)}{{\eta}\mu}, 1]$:
\begin{equation*}
\begin{aligned}
\xi_{r}(\eta) - (H_{r}(\eta) - H(\eta)) &= (\eta + d - 1) - (d\mu+1)(1-\eta)\\
&\;\;\;- d\eta + (1-\eta)(1+\mu) \\
&= (1-d)(1-\mu)(\eta - 1) \leq 0
\end{aligned}
\end{equation*}
\end{enumerate}
Thus, $\forall \mu \in (0,1]$ and $\forall \eta \in [0,1]$, we get
\begin{equation*}
{\xi}_{r}(\eta) \leq {H_{r}}(\eta) - H(\eta)
\end{equation*}
{\bf Part 3: } ${\xi}_{1}(\eta) \leq {H_{1}}(\eta) - H(\eta)$ 
$H_{1}(\eta)$ is expressed as
\begin{equation*}
\begin{aligned}
H_{1}(\eta) &= \inf_{z>\rho}\;\;{\eta}L_{dr}(z,\rho) + (1-\eta)L_{dr}(-z,\rho)
\end{aligned}
\end{equation*}
Using logic of piece-wise linearity, 
\begin{equation*}
\begin{aligned}
 H_{1}(\eta)&= \min \big[ d{\eta} + (1-\eta)d(1+\mu) + (1-\eta)(1-d),\\
 &\;\;\;(1-\eta)(1+\mu)\big] \I_{\{ \eta \leq 1-d \}} + (1-\eta)(1+\mu)\I_{\{ \third \}}
\end{aligned}
\end{equation*}
Also, 
\begin{equation*}
\begin{aligned}
{\xi}_{1}(\eta) &= (1-\eta)-{\xi}(\eta) = (1 - 2{\eta}){\I_{\{ \first \}}} \\
&+ (1-\eta-d){\I_{\{ \second \}}}
\end{aligned}
\end{equation*}
Now, we can divide $H_{1}(\eta)$ function when minimum function changes its value. 
\begin{enumerate}
\item  $\first$ and $\mu \in (0, \frac{d\eta}{(1-\eta)(1-d)})$:
\begin{equation*}
\begin{aligned}
{\xi}_{1}(\eta) - {H_{1}}(\eta) + H(\eta)& = (1-2\eta) -  d\eta \\
&\;\;\;- (1-\eta)(1+d\mu)+ \eta(1+\mu) \\
&= \mu(\eta - d) + d\eta(\mu-1) \leq 0
\end{aligned}
\end{equation*}
\item $\first$ and $\mu \in [\frac{d\eta}{(1-\eta)(1-d)}, 1]$:
\begin{equation*}
\begin{aligned}
{\xi}_{1}(\eta) - {H_{1}}(\eta) + H(\eta) &= (1-2\eta) - (1-\eta)(1+\mu) \\
&\;\;\;+ \eta(1+\mu) \\
&= -\mu(1-2\eta ) \leq 0
\end{aligned}
\end{equation*}
\item $\second$ and $\mu \in (0,\frac{d\eta}{(1-\eta)(1-d)})$:
\begin{equation*}
\begin{aligned}
{\xi}_{1}(\eta) - {H_{1}}(\eta) + H(\eta) &= (1 - \eta - d)- d\eta\\
&\;\;\;- (1-\eta)(1+d\mu) + d(1+\mu) \\
&= d\eta(\mu-1) \leq 0
\end{aligned}
\end{equation*}
\item  $\second$ and $\mu \in [\frac{d\eta}{(1-\eta)(1-d)}, 1]$:
\begin{equation*}
\begin{aligned}
{\xi}_{1}(\eta) - {H_{1}}(\eta) + H(\eta) &= (1-\eta-d) - (1-\eta)(1+\mu)\\
&\;\;\;\;+ d(1+\mu) \\
&= \mu(\eta+d-1) \leq 0
\end{aligned}
\end{equation*}
\item $\third$:
\begin{equation*}
\begin{aligned}
{\xi}_{1}(\eta) - {H_{1}}(\eta) + H(\eta)& =0-(1-\eta)(1+\mu)\\
&\;\;\;+ (1-\eta)(1+\mu)\\
&= 0
\end{aligned}
\end{equation*}
\end{enumerate}
This implies that $\forall \mu \in (0,1],\;\forall \eta \in [0,1]$
\begin{equation*}
{\xi}_{1}(\eta) \leq ({H_{1}}(\eta) - H(\eta))
\end{equation*}

\section{Proof of Proposition~4}
For $S = \{ \xx_{i}, y_{i} \}_{i=1}^{N}$ and $\lambda_{2}$, let $f_{ \lambda_{2}, S}^* = (h_{\lambda_{2}, S}^*, b_{\lambda_{2}, S}^*)$, where

\begin{equation}
\label{eq-l2-risk-minimize}
\begin{aligned}
(h_{\lambda_{2}, S}^*, b_{\lambda_{2}, S}^*, \rho_{\lambda_{2}, S}^*) = \arg\min_{h, b, \rho}\; \frac{\lambda_{2}}{2} \| h \|^2_{\K} + \hat{ \sR}_{dr}(f, \rho) 
\end{aligned}
\end{equation}

Beside the continuous optimization problem in (\ref{eq-l2-risk-minimize}), double ramp loss based optimization problem can be formed as a mixed integer optimization problem as below,
\begin{equation}
\label{eq-original-mi-op-prob}
\begin{aligned}
\min \; & \frac{\lambda_{2}}{2}\|h \|^2_{\K} + \frac{1}{N} \sum_{i=1}^N\big{(} d(e_{i, 1} + p_{i, 1}) + (1 - d) (e_{i, 2} + p_{i, 2})\big{)} \\
s.t. \;\;\; & \;\;\; p_{i, 1}, p_{i, 2} \in \{0 ,2\};\;\;i=1\ldots N\\
& \;\;\; 0 \leq e_{i, 1}, e_{i, 2} \leq 2 ;\;\;i=1\ldots N\\
& \;\;\; y_{i}(f(\xx_{i})) \geq \rho + 1 - e_{i, 1} \;\;\;\;\; \text{if } p_{i, 1} = 0;\;\;i=1\ldots N\\
& \;\;\; y_{i}(f(\xx_{i})) \geq -\rho + 1 - e_{i, 2} \;\;\; \text{if } p_{i, 2} = 0;\;\;i=1\ldots N\\
& \;\;\; \rho \geq 1 \\
\end{aligned}
\end{equation}
where $h\in {\cal H}_{\K}, b, \rho$ and $\{p_{i, 1}, p_{i, 2} , e_{i, 1}, e_{i, 2}\}_{i=1}^N$ are the optimization variables. The optimization problem (\ref{eq-original-mi-op-prob})  should be solved over all possible values of ${\bf p} = \big{(} p_{1, 1}, p_{1, 2}, p_{2, 1}, ..., p_{N, 1}, p_{N, 2} \big{)} \in \{ 0, 2\}^{2N}$. 
When the optimal vector ${\bf p}^{*}$ is given, optimization problem (\ref{eq-original-mi-op-prob}) is reduced to the following quadratic optimization problem. 

\begin{align}
\label{eq-mi-op-prob-without-p}
\nonumber &\min_{h\in {\cal H}_{\K}, b, \rho , e_{i, 1}, e_{i, 2}} \;\;\;  \frac{\lambda_{2}}{2}\| h \|^2_{\K} + \frac{1}{N} \sum_{i=1}^N\big{(} de_{i, 1} + (1 - d)e_{i, 2} \big{)} \\[5pt]
\nonumber  
& s.t.\;\; 0 \leq e_{i, 1}, e_{i, 2} \leq 2 ;\;\;i=1\ldots N\\
\nonumber & \;\;\; y_{i}(f(\xx_{i})) \geq \rho + 1 - e_{i, 1} \;\;\;\;\; \text{if } p_{i, 1}^* = 0;\;\;i=1\ldots N\\
 \nonumber& \;\;\; y_{i}(f(\xx_{i})) \geq -\rho + 1 - e_{i, 2} \;\;\; \text{if } p_{i, 2}^* = 0;\;\;i=1\ldots N\\
& \;\;\; \rho \geq 1 
\end{align}

We first show that given vector ${\bf p}^{*}$, the optimal solution of problem (\ref{eq-mi-op-prob-without-p}) is same as the optimal solution of problem (\ref{eq-mi-op-red-con}) described as follows. 
\begin{equation}
\label{eq-mi-op-red-con}
\begin{aligned}
&\min_{h\in {\cal}_{\K}, b, \rho , e_{i, 1}, e_{i, 2}} \;  \frac{\lambda_{2}}{2}\| h \|^2_{\K} + \frac{1}{N} \sum_{i=1}^N\big{(} de_{i, 1} + (1 - d)e_{i, 2} \big{)} \\[5pt] 
& s.t.\;\;  e_{i, 1}, e_{i, 2} \geq 0 ;\;\;i=1\ldots N\\
& \;\;\; y_{i}(f(\xx_{i})) \geq \rho + 1 - e_{i, 1} \;\;\;\;\; \text{if } p_{i, 1}^* = 0;\;\;i=1\ldots N\\
& \;\;\; y_{i}(f(\xx_{i})) \geq -\rho + 1 - e_{i, 2} \;\;\; \text{if } p_{i, 2}^* = 0;\;\;i=1\ldots N\\
& \;\;\; \rho \geq 1
\end{aligned}
\end{equation}
Let $(h^{*'}, b^{*'}, \rho^{*'}, e_{i, 1}^{*'}, e_{i, 2}^{*'})$ and $(h^*, b^*, \rho^*, e_{i, 1}^*, e_{i, 2}^*)$ be the optimal solutions of problems (\ref{eq-mi-op-prob-without-p}) and (\ref{eq-mi-op-red-con}) respectively. Then, it can be immediately seen that
\begin{align*}
&\frac{\lambda_{2}}{2}\|h^{*} \|^2_{\K} + \frac{1}{N} \sum_{i=1}^N\big{(} de_{i, 1}^{*} + (1 - d)e_{i, 2}^{*} \big{)} \\
&\leq \frac{\lambda_{2}}{2}\| h^{*'} \|^2_{\K} + \frac{1}{N} \sum_{i=1}^N\big{(} de_{i, 1}^{*'} + (1 - d)e_{i, 2}^{*'} \big{)}
\end{align*}
To prove our claim, we just need to verify that $0 \leq e_{i,1}^{*}, e_{i,2}^{*} \leq 2$. If $p_{i,1}^* = 2$ then $e_{i,1}^* = 0$. Similarly, if $p_{i,2}^* = 2$ then $e_{i,2}^* = 0$. Now, we prove for the case $p_{i,1}^* = 0$ and in the same manner, we can prove for the case $p_{i,2}^* = 0$. Let $I = \{i \in \{1,2,...,N\}: p_{i,1}^* = 0\text{ and } e_{i,1}^* > 2\}$. If $I$ is not empty, define a new vector ${\bf p}'$ as follows. $p'_{i,1} = 2,\;\forall i \in I$ and $p'_{i,1} = p_{i,1}^*,\;\forall i \notin I$. As $p^*_{i, 1} = 2$ implies that $e^*_{i, 1} = 0$, we define $e'_{i,1}$ as follows. $e'_{i,1} = 0$ if $p'_{i,1} = 2$ and $e'_{i,1} = e_{i,1}^*$ otherwise. Now, 
\begin{equation*}
\begin{aligned}
&\frac{\lambda_{2}}{2}\| h^{*'} \|^2_{\K} + \frac{1}{N} \sum_{i=1}^N\big{(} d(e_{i, 1}^{*'} + p_{i, 1}^{*}) + (1 - d) (e_{i, 2}^{*'} + p_{i, 2}^{*})\big{)} \\ 
& \geq 
 \frac{\lambda_{2}}{2}\|h^* \|^2_{\K}  
+ \frac{1}{N} \sum_{i=1}^N\big{(} d(e_{i, 1}^* + p_{i, 1}^*) + (1 - d) (e_{i, 2}^* + p_{i, 2}^*)\big{)} \\
& > 
\frac{\lambda_{2}}{2}\| h^* \|^2_{\K}  
+ \frac{1}{N} \sum_{i=1}^N\big{(} d(e'_{i, 1} + p'_{i, 1}) + (1 - d) (e'_{i, 2} + p'_{i, 2})\big{)} \\
\end{aligned}
\end{equation*}
Which contradicts the assumption that $(h^{*'}, b^{*'}, \rho^{*'}, e_{i, 1}^{*'}, e_{i, 2}^{*'})$ is the optimal solution of problem (\ref{eq-mi-op-prob-without-p}). Hence optimal solution of (\ref{eq-mi-op-prob-without-p}) and (\ref{eq-mi-op-red-con}) are same.

Let $I_{1, 0} = \{i:p_{i, 1}^* = 0\}$, $I_{1, 2} = \{i:p_{i, 1}^* = 2\}$, $I_{2, 0} = \{i:p_{i, 2}^* = 0\}$ and $I_{2,2} = \{i:p_{i, 2}^* = 2\}$. As $(h_{\lambda_{2}, S}^*,\; b_{\lambda_{2}, S}^*, \; \rho_{\lambda_{2}, S}^{*}, \; e_{i, 1}^{*}, \; e_{i, 2}^{*})$ is the optimal solution of problem (\ref{eq-mi-op-red-con}), it satisfies the KKT conditions as follows. 
\begin{equation}
\label{kkt-conditions}
\begin{aligned}
& h_{\lambda_{2}, S}^{*}(\xx) = \sum_{i \in I_{1, 0}} \alpha_{i, 1}^* y_{i} \K(\xx_{i}, \xx) +  \sum_{i \in I_{2, 0}}  \alpha_{i, 2}^* y_{i} \K(\xx_{i}, \xx) \\
& 0 \leq \alpha_{i,1}^* \leq \frac{d}{\lambda_{2}m};\;i\in I_{\{1,0\}} \\
& 0 \leq \alpha_{i,2}^* \leq \frac{(1-d)}{\lambda_{2}m} ;\;i\in I_{\{2,0\}}\\[6pt]
& \sum_{i \in I_{1, 0}} \alpha_{i, 1}^* y_{i} + \sum_{i \in I_{2, 0}} \alpha_{i, 2}^* y_{i} = 0 \\[6pt]
& \sum_{i \in I_{1, 0}} \alpha_{i, 1}^* - \sum_{i \in I_{2, 0}} \alpha_{i, 2}^* - \gamma = 0 \;\;\; \text{ with } \gamma \geq 0 \\
& 1 + \rho_{\lambda_2, S}^* - y_{i}f_{\lambda_2, S}^*(\xx_{i}) \leq 0 ;\; \text{if } i \in I_{1, 0}\;\&\; \alpha_{i, 1} = 0 \\
& 0 \leq e_{i,1}^* = 1 + \rho_{\lambda_2, S}^* - y_{i}f_{\lambda_2, S}^*(\xx_{i}) \leq 2;\; \text{if } i \in I_{1, 0}\;\&\; \alpha_{i, 1} \neq 0 \\
& 1 - \rho_{\lambda_2, S}^* - y_{i}f_{\lambda_2, S}^*(\xx_{i}) \leq 0;\; \text{if } i \in I_{2, 0} \;\&\; \alpha_{i, 2} = 0 \\
& 0 \leq e_{i,2}^* = 1 - \rho_{\lambda_2, S}^* - y_{i}f_{\lambda_2, S}^*(\xx_{i}) \leq 2; \;\text{if } i \in I_{2, 0} \;\&\; \alpha_{i, 2} \neq 0 \\
\end{aligned}
\end{equation}
where $\{ \alpha_{i,1}^* \}_{i \in I_{1, 0}}$, $\{ \alpha_{i,2}^* \}_{i \in I_{2, 0}}$ and $\gamma$ are the Lagrange multipliers. From (\ref{kkt-conditions}), we see that when $i \in I_{1, 0}$, \;$y_{i}f_{\lambda_2, S}^*(\xx_{i}) \geq \rho_{\lambda_2, S}^* - 1$. Thus, 
\begin{equation}
\label{eq-loss-inq1}
L_{dr}(y_{i}f_{\lambda_2, S}^*(\xx_{i}),  \rho_{\lambda_{2},S}^*) \geq 1 - y_{i}f_{\lambda_{2},S}^*(\xx_{i}) + \rho_{\lambda_{2},S}^*
\end{equation}
Similarly, we see that when $i \in I_{2, 0}$, \;$y_{i}f_{\lambda_2, S}^*(\xx_{i}) \geq -\rho_{\lambda_2, S}^*(\xx_{i}) - 1$. Thus, 
\begin{equation}
\label{eq-loss-inq2}
L_{dr}(y_{i}f_{\lambda_2, S}^*(\xx_{i}), \rho_{\lambda_2, S}^*) \geq 1 - y_{i}f_{\lambda_{2},S}^*(\xx_{i}) - \rho_{\lambda_{2},S}^*
\end{equation}
From KKT conditions, we know that $h_{\lambda_{2}, S}^{*}(\xx) = \sum_{i \in I_{1, 0}} \alpha_{i, 1}^* y_{i} \K(\xx_{i}, \xx) +  \sum_{i \in I_{2, 0}}  \alpha_{i, 2}^* y_{i} \K(\xx_{i}, \xx)$. Thus, $h_{\lambda_{2}, S}^{*} \in {\cal H}_{\K}$.
Thus, by using the definition of $\Vert . \Vert_{\K}$ norm,
$\sum_{i \in I_{1, 0}} \alpha_{i, 1}^* y_{i} h_{\lambda_{2}, S}^* (\xx_{i}) + \sum_{i \in I_{2, 0}} \alpha_{i, 2}^* y_{i} h_{\lambda_{2}, S}^* (\xx_{i})=\| h_{\lambda_2, S}^* \|^2_{\K}$. Let $\alpha_{i}^* = \alpha_{i,1}^* + {\alpha}_{i,2}^*$. Thus,
\begin{equation*}
\begin{aligned}
&\Omega(h_{\lambda_{2}, S}^* ) = \sum_{i=1}^N \alpha_{i}^*  
= \sum_{i \in I_{1,0}} \alpha_{i, 1}^* + \sum_{i \in I_{2, 0}} \alpha_{i, 2}^* \\
& = \sum_{i \in I_{1, 0}} \alpha_{i, 1}^* (1 - y_{i}\left(h_{\lambda_{2},S}^*(\xx_{i})+ b_{\lambda_{2},S}^*\right)+\rho_{\lambda_{2},S}^*)\\
&+ \sum_{i \in I_{2, 0}} \alpha_{i, 2}^* (1 - y_{i}\left(h_{\lambda_{2},S}^*(\xx_{i})+b_{\lambda_{2},S}^*\right) - \rho_{\lambda_{2},S}^*)\\
&+ \| h_{\lambda_2, S}^* \|^2_{\K} - \gamma \rho_{\lambda_{2}, S}^*
\end{aligned}
\end{equation*}

As $\gamma \geq 0$ and $\rho_{\lambda_{2}, S}^* \geq 0$, $\gamma \rho_{\lambda_{2}, S}^* \geq 0$. Using eq.(\ref{eq-loss-inq1}) and (\ref{eq-loss-inq2}), we get 
\begin{align*}
&\Omega( h_{\lambda_{2}, S}^* ) \leq \sum_{i \in I_{1, 0}} \alpha_{i, 1}^* L_{dr} (y_{i} f_{\lambda_2, S}^*(\xx_{i}), \rho_{\lambda_2, S}^*) \\
&+ \sum_{i \in I_{2, 0}} \alpha_{i, 2}^* L_{dr} (y_{i} f_{\lambda_2, S}^*	(\xx_{i}), \rho_{\lambda_2, S}^*) + \| h_{\lambda_2, S}^* \|^2_{\K} \\
&\leq \sum_{i \in I_{1, 0}}  \frac{d}{\lambda_2 m} L_{dr} (y_{i} f_{\lambda_2, S}^*(\xx_{i}), \rho_{\lambda_2, S}^*)\\
&+ \sum_{i \in I_{2, 0}}\frac{(1- d)}{\lambda_2 m} L_{dr} (y_{i} f_{\lambda_2, S}^*(\xx_{i}), \rho_{\lambda_2, S}^*) + \| h_{\lambda_2, S}^* \|^2_{\K} \\
& \leq \frac{1}{\lambda_2} \hat{\sR}_{dr}(f_{ \lambda_{2}, S}^*, \rho_{\lambda_2, S}^*) + \| h_{\lambda_2, S}^* \|^2_{\K} \\
& = \frac{1}{\lambda_2} \hat{\sR}_{dr}(h_{ \lambda_{2}, S}^*+b_{ \lambda_{2}, S}^*, \rho_{\lambda_2, S}^*) + \| h_{\lambda_2, S}^* \|^2_{\K}
\end{align*}
Hence, the bound for $\Omega(h_{\lambda_{2}, S}^* )$ follows.

\section{Proof of Theorem~5}
We know that 
$$(h_{\lambda_1, S}^*,b_{\lambda_1, S}^*,\rho_{\lambda_1, S}^*)=\arg\min_{h,b,\rho}\;\;\hat{\sR}_{dr} (h+b, \rho) + \lambda_1 \Omega (h)$$
Thus
\begin{align*}
\hat{\sR}_{dr} (f_{\lambda_1, S}^*, \rho_{\lambda_1, S}^*) + \lambda_1 \Omega (h_{\lambda_1, S}^*)& \leq \hat{\sR}_{dr} (f_{\lambda_2, S}^*, \rho_{\lambda_2, S}^*) \\
&+ \lambda_1 \Omega (h_{\lambda_2, S}^*)
\end{align*}
Using Proposition~4, we get
\begin{align*}
\nonumber &\hat{\sR}_{dr} (f_{\lambda_1, S}^*, \rho_{\lambda_1, S}^*) + \lambda_1 \Omega (h_{\lambda_1, S}^*) \leq \hat{\sR}_{dr} (f_{\lambda_2, S}^*, \rho_{\lambda_2, S}^*)\\
\nonumber &+ \lambda_1 \Big(\frac{1}{\lambda_2} \hat{\sR}_{dr}({f}_{ \lambda_{2}, S}^*, \rho_{\lambda_2, S}^*) + \|h_{\lambda_2, S}^* \|^2_{\K} \Big) \\
\nonumber &= \Big( 1 + \frac{\lambda_1}{\lambda_2} \Big) \hat{\sR}_{dr}({f}_{ \lambda_{2}, S}^*, \rho_{\lambda_2, S}^*) + \lambda_1 \| h_{\lambda_2, S}^* \|^2_{\K}\\
\end{align*}
Adding $\sR_{dr} (f_{\lambda_1, S}^*, \rho_{\lambda_1, S}^*) - \hat{\sR}_{dr} (f_{\lambda_1, S}^*, \rho_{\lambda_1, S}^*)$ both side, we get,
\begin{align}
\label{eq-thm4-left-bound}
\nonumber &\sR_{dr} (f_{\lambda_1, S}^*, \rho_{\lambda_1, S}^*) + \lambda_1 \Omega (h_{\lambda_1, S}^*) \leq \sR_{dr} (f_{\lambda_1, S}^*, \rho_{\lambda_1, S}^*) \\
\nonumber &- \hat{\sR}_{dr} (f_{\lambda_1, S}^*, \rho_{\lambda_1, S}^*) + \Big( 1 + \frac{\lambda_1}{\lambda_2} \Big) \hat{\sR}_{dr}({f}_{ \lambda_{2}, S}^*, \rho_{\lambda_2, S}^*) \\
&+ \lambda_1 \| h_{\lambda_2, S}^* \|^2_{\K}
\end{align}
Now, we will bound last two terms of the right side of the above equation using definition of ${f}_{ \lambda_{2}, S}^*$ and ${f}_{ \lambda_{2}}^*$.
\begin{equation*}
\begin{aligned}
&\Big( 1 + \frac{\lambda_1}{\lambda_2} \Big) \hat{\sR}_{dr}({f}_{ \lambda_{2}, S}^*, \rho_{\lambda_2, S}^*) + \lambda_1 \| h_{\lambda_2, S}^* \|^2_{\K} \\
&\leq \Big( 1 + \frac{\lambda_1}{\lambda_2} \Big) (\hat{\sR}_{dr}({f}_{ \lambda_{2}, S}^*, \rho_{\lambda_2, S}^*) + \lambda_2 \| h_{\lambda_2, S}^* \|^2) \\
&\leq \Big( 1 + \frac{\lambda_1}{\lambda_2} \Big) ( \hat{\sR}_{dr}({f}_{ \lambda_{2}}^*, \rho_{\lambda_2}^*) + \lambda_2 \| h_{\lambda_2}^* \|^2_{\K}) \\
&= \Big( 1 + \frac{\lambda_1}{\lambda_2} \Big) \Big{(} \hat{\sR}_{dr}({f}_{ \lambda_{2}}^*, \rho_{\lambda_2}^*) - \sR_{dr}({f}_{ \lambda_{2}}^*, \rho_{\lambda_2}^*) \\
&\;\;\; + \sR_{dr}({f}_{ \lambda_{2}}^*, \rho_{\lambda_2}^*)  + \lambda_2 \|h_{\lambda_2}^* \|^2_{\K}\Big{)}
\end{aligned}
\end{equation*}
Using this bound in eq.(\ref{eq-thm4-left-bound}), we get
\begin{equation*}
\begin{aligned}
&\sR_{dr} (f_{\lambda_1, S}^*, \rho_{\lambda_1, S}^*) + \lambda_1 \Omega (h_{\lambda_1, S}^*) \leq \sR_{dr} (f_{\lambda_1, S}^*, \rho_{\lambda_1, S}^*)\\
&- \hat{\sR}_{dr} (f_{\lambda_1, S}^*, \rho_{\lambda_1, S}^*)+ \Big( 1 + \frac{\lambda_1}{\lambda_2} \Big) \Big{(} \hat{\sR}_{dr}({f}_{ \lambda_{2}}^*, \rho_{\lambda_2}^*)\\
&- \sR_{dr}({f}_{ \lambda_{2}}^*, \rho_{\lambda_2}^*) + \sR_{dr}({f}_{ \lambda_{2}}^*, \rho_{\lambda_2}^*) + \lambda_2 \| h_{\lambda_2}^* \|^2_{\K} \Big{)} 
\end{aligned}
\end{equation*}
Bounding $\sR_{dr} (f_{\lambda_1, S}^*, \rho_{\lambda_1, S}^*) - \sR_{dr} (f_{d}^*, \rho_{d}^*) + \lambda_1 \Omega (h_{\lambda_1, S}^*)$,
\begin{equation*}
\begin{aligned}
&\sR_{dr} (f_{\lambda_1, S}^* , \rho_{\lambda_1, S}^*)  -  \sR_{dr}  (f_{d}^*, \rho_{d}^*)  + \lambda_1 \Omega (h_{\lambda_1, S}^*) \\
&\leq (\sR_{dr} (f_{\lambda_1, S}^*, \rho_{\lambda_1, S}^*) - \hat{\sR}_{dr} (f_{\lambda_1, S}^*, \rho_{\lambda_1, S}^*)) + \frac{\lambda_1}{\lambda_2} \sR_{dr} (f_{d}^*, \rho_{d}^*)\\
&+ \Big( 1+ \frac{\lambda_1}{\lambda_2}\Big) ( \hat{\sR}_{dr}({f}_{ \lambda_{2}}^*, \rho_{\lambda_2}^*) -  \sR_{dr}({f}_{ \lambda_{2}}^*, \rho_{\lambda_2}^*) ) \\
 &+ \Big( 1 + \frac{\lambda_1}{\lambda_2} \Big) (\sR_{dr}({f}_{ \lambda_{2}}^*, \rho_{\lambda_2}^*) - \sR_{dr} (f_{d}^*, \rho_{d}^*) + \lambda_2 \| h_{\lambda_2}^* \|^2_{\K} ) 
\end{aligned}
\end{equation*}
But, $\mathcal{A}(\lambda_2)=\sR_{dr}(f_{\lambda_2}^*, \rho_{\lambda_2}^*) - \sR_{dr} (f_{d}^*, \rho_{d}^*) + \lambda_2 \| h_{\lambda_2}^* \|^2_{\K}$
Thus,
\begin{align*}
&\sR_{dr} (f_{\lambda_1, S}^*, \rho_{\lambda_1, S}^*) - \sR_{dr} (f_{d}^*, \rho_{d}^*) + \lambda_1 \Omega (h_{\lambda_1, S}^*) \\
&\leq \psi \sR_{dr} (f_{d}^*, \rho_{d}^*) + \mathcal{S}(N, \lambda_1, \lambda_2) + 2 \mathcal{A}(\lambda_2)
\end{align*}

\section{Proof of Theorem~6}
The proof of Theorem~6 requires 2 more results described 
in Lemma~6a and Lemma~6b. We first discuss these two Lemmas and then discuss the proof of Theorem~6.

\begin{lemma} 
\label{lemma:abc}Consider the optimization problem associated with SDR-SVM as follows.
\begin{align}
\label{eq-sdr-svm}
\min_{h\in {\cal H}_{\K}^+,b,\rho}\;\;\hat{\sR}_{dr} (h+b, \rho) + \lambda_1 \Omega (h)
\end{align}
There exists a solution $(h_{\lambda_1, S}^*, b_{\lambda_1, S}^*, \rho_{\lambda_1, S}^*)$ for above
which satisfies $\min_{1 \leq i \leq N} |f_{\lambda_1, S}^*| \leq \rho_{\lambda_1, S}^* + 1$ and hence, 
$$| b_{\lambda_1, S}^* | \leq 1 + \| h_{\lambda_1, S}^* \|_{\K} + \rho_{\lambda_1, S}^*$$  
\end{lemma}
\begin{proof}
Let $(h_{\lambda_1, S}^*, b_{\lambda_1, S}^*, \rho_{\lambda_1, S}^*)$ be such that
\begin{equation*}
r := \min_{1 \leq i \leq N} | f_{\lambda_1, S}^* (\xx_{i}) | = | f_{\lambda_1, S}^* (\xx_{i_{0}}) | > \rho_{\lambda_1, S}^* + 1
\end{equation*}
Then for each $i$, $y_{i} f_{\lambda_1, S}^*(\xx_{i}) \geq r > \rho_{\lambda_1, S}^* + 1$ or $y_{i} f_{\lambda_1, S}^*(\xx_{i}) \leq -r < - \rho_{\lambda_1, S}^* - 1$. Now, define a new function $g(x) = f_{\lambda_1, S}^*(\xx_{i}) - (r - \rho_{\lambda_1, S}^* - 1)\text{sgn}(f_{\lambda_1, S}^*(\xx_{i_0}))$. When $y_{i}f_{\lambda_1, S}^*(\xx_{i}) > \rho_{\lambda_1, S}^* + 1$, we can check that $y_{i}g(\xx_{i}) \geq \rho_{\lambda_1, S}^* + 1$. Similarly, when $y_{i}f_{\lambda_1, S}^*(\xx_{i}) <  -\rho_{\lambda_1, S}^* - 1$, we will have $y_{i}g(\xx_{i}) \leq -\rho_{\lambda_1, S}^* - 1$. Because of above two facts, we can say that $\hat{\sR}_{dr}(f_{\lambda_1, S}^*, \rho_{\lambda_1, S}^*) = \hat{\sR}_{dr}(g, \rho_{\lambda_1, S}^*)$. Here, $ | g(\xx_{i_{0}}) | = \rho_{\lambda_1, S}^* + 1 $ therefore it satisfies our condition. Therefore, $g$ is also a solution of problem (\ref{eq-sdr-svm}) and satisfies required condition. 

Now, if $(h_{\lambda_1, S}^*, b_{\lambda_1, S}^*, \rho_{\lambda_1, S}^*)$ satisfies the condition, we get,
\begin{equation*}
\min_{1 \leq i \leq N} | f_{\lambda_1, S}^* (\xx_{i}) | = | f_{\lambda_1, S}^* (\xx_{i_0}) |  \leq \rho_{\lambda_1, S}^* + 1
\end{equation*}
then we have
\begin{equation*}
\begin{aligned}
| b_{\lambda_1, S}^* | - & | h_{\lambda_1, S}(\xx_{i_0}) | \leq \rho_{\lambda_1, S}^* + 1 \\
| b_{\lambda_1, S}^* | - \rho_{\lambda_1, S}^* \leq & | h_{\lambda_1, S}(\xx_{i_0}) | +  1 \leq 1 + \| h_{\lambda_1, S}^* \|_{\infty}
\end{aligned}
\end{equation*}
In this way, we can get required bound on $| b_{\lambda_1, S}^* |$.
\end{proof}

\begin{lemma}
\label{lemma-norm-bound}
For every $\lambda_1 > 0$, we have
\begin{equation*}
\| h_{\lambda_1, S}^* \|_{\K} \leq \tau \Omega (h_{\lambda_1, S}^*) \leq \frac{ 2d \tau}{\lambda_1}
\end{equation*}
where $\tau = \sup_{\xx, {\bf y} \in \X} \sqrt{| \K (\xx, {\bf y}) |}$.
\end{lemma}
\begin{proof}
Using representer theorem, $h_{\lambda_1, S}(\xx) = \sum_{i=1}^N \alpha_{i}^* y_{i} \K(x_{i}, x) $. Therefore, 
\begin{equation*}
\| h_{\lambda_1, S}^* \|_{\K} = \Big( \sum_{i, j = 1}^N \alpha_{i}^* \alpha_{j}^* y_iy_j\K(\xx_{i}, \xx_{j}) \Big)^{1/2} 
\end{equation*}
Using definition of $\tau$, 
\begin{equation}
\label{eq-normw-omega-relation}
\| h_{\lambda_1, S}^* \|_{\K}	\leq \tau  \Big( \sum_{i, j = 1}^N \alpha_{i}^* \alpha_{j}^* \Big)^{1/2} = \tau \Omega (h_{\lambda_1, S}^*)
\end{equation}
Using definition of $h_{\lambda_1, S}^*$, we have
\begin{equation*}
\hat{\sR}_{dr} (f_{\lambda_1, S}^*, \rho_{\lambda_1, S}^*) + \lambda_1 \Omega (h_{\lambda_1, S}^*) \leq \hat{\sR}_{dr} (0, \rho_{\lambda_1, S}^*) + \lambda_1 \Omega (0) \leq 2d
\end{equation*}
This gives $\Omega (h_{\lambda_1, S}^*) \leq \frac{2d}{\lambda_1}$. Now, using the eq.(\ref{eq-normw-omega-relation}),
\begin{equation}
\label{eq-w-bound}
\| h_{\lambda_1, S}^* \|_{\K} \leq \frac{2d \tau }{\lambda_1} 
\end{equation}
\end{proof}

We use Lemma~\ref{lemma:abc}, Lemma~\ref{lemma-norm-bound} and the fact that $\|  h_{\lambda_1, S}^* \|_{\infty} \leq \tau \|  h_{\lambda_1, S}^* \|_{\K},\;\forall h\in {\cal H}_{\K}$.
We can say that 
\begin{equation}
\label{eq-b-bound}
| b_{\lambda_1, S}^* | - \rho_{\lambda_1, S}^* \leq 1 + \frac{2d \tau^2}{\lambda_1}
\end{equation}

Now we discuss the proof of Theorem~6.

\subsection{\bf Proof of Theorem~6}
\begin{proof}
Define a random variable $\zeta_{i} = L_{dr} (y_{i}f_{\lambda_2}^*(\xx_{i}), \rho_{\lambda_2}^*)$. Now, we will estimate $\hat{\sR}_{dr}(f_{\lambda_2}^*, \rho_{\lambda_2}^*) - \sR_{dr}(f_{\lambda_2}^*,, \rho_{\lambda_2}^*)$ using random variable $\zeta_{i}$. Here, note that $\zeta_{i} = L_{dr} (y_{i}f_{\lambda_2}^*(\xx_{i}), \rho_{\lambda_2}^*) \in [0, 2]$ therefore using Hoeffding's inequality, with probability $1 - {\delta} / {2}$,
\begin{equation}
\label{eq-sample-bound1}
\hat{\sR}_{dr}(f_{\lambda_2}^*, \rho_{\lambda_2}^*) - \sR_{dr}(f_{\lambda_2}^*, \rho_{\lambda_2}^*) \leq \sqrt{ \frac{ 2 \text{log} \frac{2}{\delta}}{N} }
\end{equation}
As $f_{\lambda_1, S}^*$ varies with samples, the term $\sR_{dr} (f_{\lambda_1, S}^*, \rho_{\lambda_1, S}^*) - \hat{\sR}_{dr} (f_{\lambda_1, S}^*, \rho_{\lambda_1, S}^*)$ can not be bound in the same manner. In order to bound the above term, we shall use the concentration inequalities to the function space. We can directly use results of \cite{Bartlett:2003:RGC:944919.944944} to deal with this term therefore the following inequality holds with probability at least 1 - $\delta / 2$, 
\begin{equation}
\label{eq-sample-bound2}
\begin{aligned}
\sR_{dr} & (f_{\lambda_1, S}^*, \rho_{\lambda_1, S}^*) - \hat{\sR}_{dr} (f_{\lambda_1, S}^*, \rho_{\lambda_1, S}^*) \\ 
& \leq \mathbb{E}_{S} \mathbb{E}_{\sigma} \Big[ \sup_{g_{\rho} \in \mathcal{F}_{l}  } \Big| \frac{2}{N} \sum_{i=1}^{N} \sigma_{i} g_{\rho}(y_i,\xx_{i}) \Big|  \Big] + \sqrt{ \frac{8 \; \text{log} \frac{4}{\delta}}{N} }
\end{aligned}
\end{equation}
where $\mathcal{F}_{l} := \{ (\xx, y) \rightarrow L_{dr}(yf(\xx), \rho) - L_{dr}(0, \rho) : f \in \mathcal{F} \}$ and $\sigma_{1}, \sigma_{2}, ... ,\sigma_{N}$ are Rademacher random variables. As the Double Ramp Loss is Lipschitz with constant 1, we further bound the first term in the right-hand side by the result of \cite{Bartlett:2003:RGC:944919.944944}, 
\begin{equation*}
\begin{aligned}
&\mathbb{E}_{S}  \mathbb{E}_{\sigma} \Big[ \sup_{g_{\rho} \in \mathcal{F}_{l}  }  \Big| \frac{2}{N}  \sum_{i=1}^{N} \sigma_{i} g_{\rho}(y_i,\xx_{i}) \Big|  \Big]  \\
& \leq 2 \mathbb{E}_{S} \mathbb{E}_{\sigma} \left[ \sup_{f_{\lambda_1, S}^*\in \mathcal{F}, \rho_{\lambda_1, S}^*   } \left| \frac{2}{N} \sum_{i=1}^{N} \sigma_{i} \big( | f_{\lambda_1, S}^*(\xx_{i}) | - \rho_{\lambda_1, S}^* \big) \right|  \right] \\
& \leq 2 \mathbb{E}_{S} \mathbb{E}_{\sigma} \left[ \sup_{f_{\lambda_1, S}^*, \rho_{\lambda_1, S}^*   } \left| \frac{2}{N} \sum_{i=1}^{N} \sigma_{i} \big( | h^*_{\lambda_1, S}(\xx_{i}) | + | b_{\lambda_1, S}^* | - \rho_{\lambda_1, S}^* \big) \right|  \right] \\
& \leq 2 \mathbb{E}_{S} \mathbb{E}_{\sigma} \left[ \sup_{f_{\lambda_1, S}^*, \rho_{\lambda_1, S}^*   } \left| \frac{2}{N} \sum_{i=1}^{N} \sigma_{i}  | h^*_{\lambda_1, S}(\xx_{i}) |  \right|  \right] \\
& \;\;\;\;\; + 2 \mathbb{E}_{S} \mathbb{E}_{\sigma} \left[ \sup_{f_{\lambda_1, S}^*, \rho_{\lambda_1, S}^*   } \left| \frac{2}{N} \sum_{i=1}^{N}  \sigma_{i} \big( | b_{\lambda_1, S}^* |  - \rho_{\lambda_1, S}^* \big) \right|  \right] \\
& \leq 2 \mathbb{E}_{S} \mathbb{E}_{\sigma} \left[ \sup_{f_{\lambda_1, S}^*, \rho_{\lambda_1, S}^*   } \left| \frac{2}{N} \sum_{i=1}^{N} \sigma_{i}  \| h_{\lambda_1, S}^* \|_{\infty}  \right|  \right] \\ 
& \;\;\;\;\; + 2 \mathbb{E}_{S} \mathbb{E}_{\sigma} \left[ \sup_{f_{\lambda_1, S}^*, \rho_{\lambda_1, S}^*   } \left| \frac{2}{N} \sum_{i=1}^{N}  \sigma_{i} \big( | b_{\lambda_1, S}^* |  - \rho_{\lambda_1, S}^* \big) \right|  \right] \\
& \leq  \frac{8d \tau^2}{ \lambda_1 \sqrt{N} } + \frac{4}{\sqrt{N}} \left( 1 + \frac{2d \tau^2}{\lambda_1} \right) \\[6pt]
& \leq \frac{16d \tau^2}{ \lambda_1 \sqrt{N} } + \frac{4}{\sqrt{N}}
\end{aligned}
\end{equation*}
Here, the last inequality follows from Lemma~6a, Lemma~6b and inequality $\mathbb{E}|g| \leq \left( \mathbb{E} g^2 \right)^{1/2}$ for any function $g$. Now, combining the above bound and eq.(\ref{eq-sample-bound1}), (\ref{eq-sample-bound2}), we have with probability at least $1 - \delta$, 
\begin{equation*}
\begin{aligned}
\mathcal{S} (N, \lambda_1, \lambda_2) & \leq \left( 1+ \psi  \right) \sqrt{ \frac{2 \; \text{log} \frac{2}{\delta}}{N} } + \sqrt{ \frac{8 \; \text{log} \frac{4}{\delta}}{N} } + \frac{16d \tau^2}{ \lambda_1 \sqrt{N} } + \frac{4}{\sqrt{N}} \\
& \leq \left( 2+ \psi  \right) \sqrt{ \frac{8 \; \text{log} \frac{4}{\delta}}{N} } + \frac{16d \tau^2}{ \lambda_1 \sqrt{N} } + \frac{4}{\sqrt{N}} \\
\end{aligned}
\end{equation*}
Let $\lambda_1 = N^{- \frac{\beta + 1}{4\beta + 1}}$ and $\lambda_2 = N^{- \frac{1}{ 4\beta + 2 }}$. Then $\psi = \frac{\lambda_1}{\lambda_2} = N^{-\frac{\beta}{4 \beta + 2}} \leq 1$. Using relations $N^{-\frac{1}{2}} \leq N^{-\frac{\beta}{4 \beta + 2}}$  and $\lambda_1 \sqrt{N} = N^{\frac{\beta}{4 \beta + 2}}$, $\mathcal{S}(N, \lambda_1, \lambda_2)$ can be bound as following with probability at least $1 - \delta$,
\begin{equation}
\label{eq-final-S-bound}
\begin{aligned}
\mathcal{S}(N, \lambda_1, \lambda_2) &\leq \left( 12\sqrt{  \text{log} \frac{4}{\delta}} + 16d \tau^2 + 4 \right) N^{-\frac{\beta}{4\beta + 2}} \\
& \leq \left( 16 + 16d \tau^2 \right) \left( \; \text{log} \frac{4}{\delta} \right)^{1/2} N^{-\frac{\beta}{4\beta + 2}}
\end{aligned}
\end{equation}
Now, putting the value of $\lambda_2$ in $\mathcal{A}(\lambda_2)$,
\begin{equation}
\label{eq-final-Approx-bound}
\mathcal{A}(\lambda_2) \leq c_{\beta} N^{-\frac{\beta}{4 \beta + 2}}
\end{equation}
Now using Theorem 3, Theorem 5 and eq. (\ref{eq-final-S-bound}), (\ref{eq-final-Approx-bound}), $\sR_{d}(f_{\lambda_1, S}^*, \rho_{\lambda_1, S}^*) - \sR_{d}(f_{d}^*, \rho_d^*)$ can be bound as following with probability at least $1 - \delta$,
\begin{equation}
\sR_{d}(f_{\lambda_1, S}^*, \rho_{\lambda_1, S}^*) - \sR_{d}(f_{d}^*, \rho_{d}^*) \leq \tilde{c} \left(  \text{log} \frac{4}{\delta} \right)^{1/2} N^{-\frac{\beta}{4\beta + 2}}
\end{equation}
where $\tilde{c} = 2c_{\beta} + 16d\tau^2 + 17$. This completes the proof. 
\end{proof}

\end{document}